\documentclass[11pt]{article}
\usepackage[margin=1in]{geometry}
\usepackage{amsmath,amssymb,mathtools,bm,booktabs}
\usepackage{amsthm}
\usepackage{verbatim}
\usepackage{graphicx}
\usepackage{multirow,enumitem}
\usepackage[numbers,sort&compress]{natbib}
\usepackage[dvipsnames]{xcolor}
\usepackage[utf8]{inputenc}
\usepackage[T1]{fontenc}
\usepackage{microtype}
\usepackage{hyperref}
\usepackage{url}

\usepackage{amsmath,amsfonts,bm}

\def\eqref#1{equation~\ref{#1}}
\def\1{\bm{1}}

\DeclareMathAlphabet{\mathsfit}{\encodingdefault}{\sfdefault}{m}{sl}
\SetMathAlphabet{\mathsfit}{bold}{\encodingdefault}{\sfdefault}{bx}{n}

\newcommand{\Cov}{\mathrm{Cov}}
\newcommand{\independent}{\mathrel{\text{\scalebox{1.07}{$\perp\mkern-10mu\perp$}}}}
\DeclareMathOperator{\DO}{do}

\hypersetup{
    colorlinks=true,
    linkcolor=Maroon,   
    urlcolor=NavyBlue,
    citecolor=Green,
}

\title{TabCF: Distributional Control Function Estimation with Tabular Foundation Models}

\author{%
  Geping Chen \\
  Iowa State University \\
  \texttt{gepingc@iastate.edu}
  \and
  Chunlin Li \\
  University of Virginia \\
  \texttt{chunlin@virginia.edu}
  \and
  Tianzhong Yang \\
  University of Minnesota \\
  \texttt{yang3704@umn.edu}
  \and
  Zhengyuan Zhu \\
  Iowa State University \\
  \texttt{zhuz@iastate.edu}
  \and
  Jing Zhou \\
  University of Manchester \\
  \texttt{jing.zhou@manchester.ac.uk}
}
\date{}

\newtheorem{proposition}{Proposition}[section]

\begin{document}

\maketitle
\begin{abstract}

Instrumental variable (IV) and control function (CF) methods are powerful tools for causal effect estimation in the presence of unmeasured confounding, yet most existing approaches target only mean effects and/or demand substantial fitting and tuning effort. In this paper, we introduce a simple method, TabCF, for control function regression using tabular foundation models, which enables accurate, fast, identification-transparent, and tuning-light causal estimation of distributional quantities, such as interventional means and quantiles; we also propose a copula-based approximation for multivariate outcomes. TabCF performs favorably against representative methods across a broad range of small- to medium-sized synthetic and real data scenarios. The central message is two-fold: for practitioners, it highlights that TabCF is an effective tool for distributional causal inference; for researchers, it suggests that the proposed approach could be considered a strong baseline for future method development. Code is available at \url{https://github.com/GepingChen/TabCF}.

\end{abstract}

\section{Introduction}
\label{sec:introduction}

Instrumental variables (IV) and control functions (CF) are powerful tools tackling unmeasured confounding in causal inference, but a majority of the IV literature has focused on mean effects ~\citep{imbens1994local,angrist1996identification,neweypowell2003npiv,guo2016control,hartford2017deepiv,chen2024discovery}. In many applications, however, the object of interest is distributional: practitioners want to know how an intervention changes risk, tail behavior, inequality, or other heterogeneous features of the outcome distribution, not only its average. 

Recent work has made remarkable progress on distributional IV estimation, including methods that target the full interventional distribution~\citep{kookpfister2025dive,holovchak2025div}. Meanwhile, many flexible IV estimators, including those mentioned above, remain expensive to tune and often require dataset-specific architecture choices, hyperparameter search, or repeated refitting for different estimands~\citep{hartford2017deepiv,bennett2019deepgmm,kookpfister2025dive,holovchak2025div}. More recently, advances in tabular foundation models (TFMs) have attracted considerable attention. These models amortize learning across many tasks and produce predictive distributions in a single forward pass~\citep{mueller2022pfn}, showing strong performance on small- to medium-sized samples and calibrated predictive uncertainty without per-dataset training~\citep{hollmann2025tabpfn, grinsztajn2025tabpfn25}. However, existing causal methods based on TFMs either focus only on mean effects \citep{balazadeh2025causalpfn,ma2025foundation} or rely on agnostic identification approaches \citep{robertson2025dopfn}.

In response to this practical gap, we introduce \textbf{TabCF}, showing that the explicit control function identification formula of \cite{imbensnewey2009triangular} can be operationalized with modern TFMs to produce an empirically accurate, computationally fast, and tuning-light estimator for interventional distribution functionals under continuous treatment IV settings.
Specifically, we summarize the strengths of TabCF in the following aspects:
\begin{enumerate}
    
    \item \textbf{TabCF delivers strong performance out of the box.} By design, TabCF inherits the strengths of modern TFMs \citep{hollmann2025tabpfn,qu2026tabiclv2}, and delivers accurate and fast distributional estimation while requiring minimal tuning. We also provide a user-friendly software for TabCF, publicly available at \url{https://github.com/GepingChen/TabCF}.
    
    \item \textbf{TabCF supports additional practical flexibility.} 
    Existing TFMs for causal inference largely focus on univariate outcomes \citep{balazadeh2025causalpfn,robertson2025dopfn,ma2025foundation}. To enable causal analysis of multiple outcomes, we have extended TabCF based on a practical copula to approximate joint interventional distributions, thereby greatly expanding its applicability. Moreover, the TabCF interface also allows the adjustment of pretreatment covariates, a feature that is practically necessary yet overlooked in many existing implementations \citep{bennett2019deepgmm,saengkyongam2022hsicx,kookpfister2025dive}.

    \item \textbf{TabCF is transparent in causal analysis.} TabCF adopts an identification approach that facilitates direct estimation using TFMs. This design separates the role of causal assumptions from the choice of the predictive backbone model. As a result, it preserves the interpretability of causal analysis, allows diagnostics of each step, and unleashes the power of using a foundation model as a plug-in distributional regressor.

    % TabCF directly targets the interventional distribution, which yields coherent estimation of means and quantiles and provides a natural entry point for covariate adjustment, multivariate intervention distributions, and quantile-calibration analysis.
\end{enumerate}

We have conducted comprehensive experiments: TabCF performs favorably against representative and state-of-the-art methods across a broad range of small- to medium-sized data scenarios. The central message is two-fold: for practitioners, TabCF is an effective tool for distributional causal inference; for researchers, we believe TabCF will serve as a strong baseline for future method development.

\paragraph{Related work.}
Our work connects TFMs and causal inference. 

TFMs have received growing attention in recent years, due to their strong performance in broad scenarios. Notably, the TabArena leaderboard~\citep{erickson2025tabarena} displays several TFMs in its top ranks: the TabPFN family, including TabPFNv1~\citep{hollmann2022tabpfn}, TabPFNv2~\citep{hollmann2025tabpfn}, and TabPFNv2.5~\citep{grinsztajn2025tabpfn25}; the TabICL family, including TabICLv1~\citep{qu2025tabicl} and TabICLv2~\citep{qu2026tabiclv2}; TabDPT~\citep{ma2025tabdpt}; and Mitra~\citep{zhang2025mitra}. Since TabCF targets interventional distributions, distributional prediction is especially relevant: TabPFNv2 and TabPFNv2.5 model regression through binned predictive distributions, while TabICLv2 directly predicts many quantiles and reconstructs distributions, densities, and moments. By contrast, TabDPT and Mitra are less direct distributional backbones in their current forms. 

Besides general TFMs, amortized causal inference also gathers significant interest. 
BBCI~\citep{bynum2025BBCI} learns a task-specific estimator by meta-training a predictor on dataset-effect pairs, enabling amortized estimation of causal effects; however, their implementation is not publicly available.
CausalPFN~\citep{balazadeh2025causalpfn} targets average causal effects while assuming unconfoundedness, so it is not designed for IV settings with hidden confounding. Do-PFN~\citep{robertson2025dopfn} predicts conditional interventional distributions without requiring the ground truth causal graph, but its identification logic is absorbed into the learned posterior, making it less transparent. CausalFM~\citep{ma2025foundation} trains PFN-based estimators for various identification regimes, but focuses only on average/conditional average treatment effects rather than interventional distributions. 
Moreover, their implementations \citep{balazadeh2025causalpfn,robertson2025dopfn,ma2025foundation} are limited to binary treatments.

\section{TabCF framework}
\label{sec:method}

We consider a structural causal model with unobserved confounding,
\begin{equation}
Y = g(X, W,\varepsilon), \qquad X = h(Z, W,\eta), 
\label{eq:triangular_scm}
\end{equation}
where $Y$ is an outcome, $X$ is a continuous treatment, $Z$ is an instrument, and $W$ is a vector of observed pretreatment covariates. The disturbances
$\varepsilon$ and $\eta$ capture unobserved factors affecting the outcome and
treatment, respectively. 
The absence of $Z$ from the outcome equation encodes
the exclusion restriction \citep{angrist1996identification}: $Z$ affects $Y$ only through $X$. We allow $(\eta,\varepsilon)$ to be arbitrarily dependent conditional on $W$. Thus, treatment assignment may remain confounded after
adjusting for $W$, because unobserved factors influencing $X$ may also
influence $Y$ \citep{imbensnewey2009triangular}.
Our target is the \emph{interventional} outcome cumulative
distribution function (CDF):
\begin{equation}
F_{Y(x)}(y) := \mathbb{P}\!\left(Y \le y \mid \DO(X=x)\right), \qquad y\in\mathbb{R},
\label{eq:Fx_def}
\end{equation}
which fully characterizes distributional causal effects, where $\DO(\cdot)$ is the do-operator \citep{pearl2009causality}. Common causal functionals can be derived from this CDF, such as:

(a) {\it Interventional mean:} $\mu (x) = \mathbb{E}\!\left[Y \mid \DO(X=x)\right] = \int_0^\infty (1 - F_{Y(x)}(y)) dy - \int_{-\infty}^0 F_{Y(x)}(y)dy$. 
 
(b) {\it Interventional $\tau$-quantile:} $q_\tau(x) = \inf\{y\in\mathbb{R}: F_{Y(x)}(y) \ge \tau\}$ for any $\tau\in(0,1)$.

(c) {\it Interventional Gini index:} $\mu (x)^{-1} \int_{0}^\infty F_{Y(x)}(y)(1 - F_{Y(x)}(y))dy$, $Y(x)\geq 0$ almost surely.

\subsection{Preliminaries}
\label{sec:cf}

\paragraph{Tabular foundation model (TFM) and in-context distributional learning.}

When $\varepsilon$ and $\eta$ are independent (i.e., no hidden confounding), the interventional distribution $F_{Y(x)}(y)$ reduces to the conditional distribution $F_{Y\mid X}(y\mid x)$. In such situations, modern TFMs \citep{hollmann2025tabpfn,qu2026tabiclv2} can provide fast, accurate estimation of conditional distributions via in-context learning. 
Given a dataset $\mathcal D$ and a test point $x$, rather than performing per-dataset training to learn the conditional distribution, TFMs learn a direct map
$(x,\mathcal{D})\mapsto F_{Y\mid X}(y\mid x)$ by 
pretraining a transformer on a large collection of synthetic datasets generated from a flexible structural prior.
At inference time, the pretrained model parameters are fixed: the model receives the entire observed dataset as ``context'' and returns $\widehat{F}_{Y\mid X}(\cdot \mid x)$ for any test point $x$ in a single forward pass, without gradient-based fitting on $\mathcal{D}$. The procedure does not require explicit tuning and has demonstrated strong performance on small-to-medium samples.

However, the reduction from $F_{Y(x)}$ to $F_{Y\mid X}$ relies on the absence of hidden confounding. When $\varepsilon$ and $\eta$ are dependent, $F_{Y\mid X}$ generally differs from the interventional distribution $F_{Y(x)}$, so directly estimating the conditional distribution is no longer sufficient 
for causal inference. To target $F_{Y(x)}$ in this setting, %To handle the dependence of $\varepsilon$ and $\eta$, 
we integrate TFMs with a control function approach. 

\paragraph{Instrumental variable and control function.} Instrumental variable (IV) is a special type of variable that helps mitigate hidden confounding. In \eqref{eq:triangular_scm}, $Z$ is called an IV if it satisfies:
\begin{enumerate}[label=\textbf{(IV\arabic*)}, ref=IV\arabic*, leftmargin=*]
\item \label{iv: relevance}
\textit{Relevance:} $X$ must depend on $Z$, or equivalently $Z\not\independent X \mid W$.
\item \label{iv: exclusion}
\textit{Exclusion:} $Z$ affects $Y$ only through $X$, or equivalently $Y\independent Z \mid X, W, \varepsilon$.
\item \label{iv: exogeneous}
\textit{Unconfoundedness:} $Z\independent (\varepsilon, \eta) \mid W$.
\end{enumerate}
With a valid IV, causal effects can be identified under suitable assumptions \citep{angrist1996identification}. For example, in the linear causal model, the IV approach is called the two-stage least squares (2SLS) regression, where the mean coefficient is identified as $\Cov(Y, Z \mid W) / \Cov (X, Z\mid W)$.

The control function (CF) approach is a subfamily of IV-based methods \citep{terza20082SRI,guo2016control,puli2020gcf}. The idea of CF is to construct a function of $(X,Z,W)$ that is independent of the instrument $Z$ conditional on $W$ and controls for the latent variation in $X$ that is related to the outcome disturbance. 
As a result, one can use CF as a proxy for the unmeasured confounding.
In the linear causal model, CF is also called the two-stage residual inclusion (2SRI) regression \citep{terza20082SRI}, and is mathematically equivalent to 2SLS.

\subsection{Choice of CF identification strategy}\label{sec:identification}

Not all CF methods, together with their identification conditions, are equally convenient to integrate with TFMs. In the literature, identification schemes based on CF can be informally categorized into two groups: the ``explicit approach'' and the ``implicit approach.''
In the explicit approach, the target causal quantity is represented directly as a functional of observed data objects \citep{florens2008control,imbensnewey2009triangular}. In this case, estimation is naturally plug-in: one estimates the required nuisance components from the data and substitutes them into the explicit identification formula. In the implicit approach, by contrast, the target causal quantity is not expressed in closed form as a direct functional of observed quantities. Instead, it is characterized as the solution to an equation, an optimization problem, a fixed point, or an inverse problem \citep{lee2007endogeneity,puli2020gcf}. While both approaches are valid, we find the explicit approach preferable when integrating TFMs, as the implicit approach may amplify estimation errors introduced by the foundation model; see Appendix~\ref{app:explicit-implicit-cf} for further discussion.

In this work, we take the explicit approach of \cite{imbensnewey2009triangular} and impose the following identifiability conditions for the control function 
\begin{equation}
    V=F_{X\mid Z,W}(X\mid Z,W).
    \label{eq:control}
\end{equation}

\begin{enumerate}[label=\textbf{(CF\arabic*)}, ref=CF\arabic*, leftmargin=*]
\item\label{CF: monotone} \textit{Monotonicity:}
    The conditional CDF $F_{\eta\mid W}(\cdot\mid 
    w)$ is strictly monotonic on its support.
    For $F_W$-almost every $w$, the map $\eta \mapsto h(z,w,\eta)$ is continuous and strictly
    monotonic, where the orientation does not depend on $z$. 

\item\label{CF: common support} \textit{Common support:}
    For $F_W$-almost every $w$ and for all $x \in \operatorname{supp}(X)$, we have $\operatorname{supp}(V\mid X=x,W=w)=\operatorname{supp}(V\mid W=w)=[0,1]$, where $\operatorname{supp}(\cdot)$ is the support set of a distribution.
\end{enumerate}
% In the above, the probability integral transform yields the control variable
% \begin{equation}
% V := F_{X\mid Z,W}(X\mid Z,W),
% \label{eq:V_def}
% \end{equation}
% which satisfies two key properties:
% (i) $V\independent Z\mid W$ and $V\mid W\sim \mathrm{Unif}(0,1)$;
% (ii) $X = F^{-1}_{X\mid Z,W}(V\mid Z,W)$ almost surely, where $F^{-1}_{X\mid Z,W}(\cdot\mid z,w)$
% denotes the generalized inverse of the conditional CDF.

Under (\ref{CF: monotone}), $V$ is a one-to-one transform of $\eta$. Therefore, \eqref{eq:triangular_scm} and (\ref{iv: exogeneous}) imply $\varepsilon\independent X \mid V, W$. Meanwhile, (\ref{CF: common support}) ensures the conditional law $Y\mid (X=x,V=v,W=w)$ is defined for all relevant $(x,v,w)$. As a result, conditioning on $(X,V,W)$ isolates the latent heterogeneity for outcomes, and the conditional CDF $F_{Y\mid X,V,W}(\cdot\mid x,v,w)$ can help recover interventional distributions by integrating out $V$. 
This is summarized in Proposition~\ref{proposition:identification}.

\begin{proposition}\label{proposition:identification}
    Suppose (\ref{iv: relevance})--(\ref{iv: exogeneous}) and (\ref{CF: monotone})--(\ref{CF: common support}) hold. Then for $(x, y) \in \operatorname{supp}(X)\times \mathbb R$, the interventional conditional and marginal CDFs are respectively identified as
    \begin{equation}
        \begin{aligned}
        F_{Y(x)\mid W}(y\mid w)= \int_0^1 F_{Y\mid X,V,W}(y\mid x,v,w)\, dv, \qquad
        F_{Y(x)}(y)
        =
        \mathbb{E}_{W}\left[ F_{Y(x)\mid W}(y \mid W)\right].
        \end{aligned}
        \label{eq:cf_identification}
    \end{equation}
\end{proposition}

% \paragraph{Remarks.}
% (i) Part (a) is a standard probability integral transform (PIT) statement implied by the
% monotonic first-stage in triangular models: $V$ is (a monotone transform of) the latent
% first-stage disturbance rank and is therefore uniform and independent of the instrument.

% (ii) \eqref{eq:cf_identification} expresses the interventional CDF as an average of the observable
% conditional CDF $F_{Y\mid X,V}(\cdot\mid x,v)$ over the marginal distribution of $V$.
% Assumption (CF4) is a common-support/overlap condition (cf. the ``common support'' assumption in
% \citealp{imbensnewey2009triangular} and the motivating continuous-treatment identification discussion in
% \citealp{holovchak2025div}): it ensures that for each target level $x$, the conditional law
% $Y\mid (X=x,V=v)$ is well-defined for all $v$ in the marginal support of $V$, making
% \eqref{eq:cf_identification} operational. When overlap fails, the interventional distribution under $do(X=x)$
% is generally only partially identified without additional restrictions.

% (iii) Consequently, estimation reduces to two conditional distribution learning problems:
% (i) estimating $F_{X\mid Z}$ to construct the control variable $V$, and (ii) estimating
% $F_{Y\mid X,V}$ and integrating out $V$ as in \eqref{eq:cf_identification}.

\subsection{TabCF Procedure}
\label{sec:tabpfn}

We now turn to estimation and let $\mathcal{D}_n = \{(W_i,Z_i, X_i, Y_i)\}_{i=1}^n$
denote the observed sample of $n$ i.i.d.\ copies of $(W,Z,X,Y)$.
TabCF produces the distributional estimate in the following three stages:

\paragraph{Stage U1: constructing the control variable.}
We treat $(Z,W)\mapsto X$ as a distributional regression task with context dataset $\mathcal{D}^{(1)}_n = \{((Z_i,W_i),X_i)\}_{i=1}^n$. 
TabCF yields an estimated  distribution $\widehat{F}_{X\mid Z,W}(\cdot\mid z,w)$ for any test point $(z,w)$. We then form the plug-in
control variable by the probability integral transform,
$\widehat{V}_i = \widehat{F}_{X\mid Z,W}(X_i\mid Z_i,W_i)$, $i=1,\ldots,n$.

\paragraph{Stage U2: estimating the conditional outcome distribution.}
Next, we treat $(X,V,W)\mapsto Y$ as a second distributional regression task with context dataset $\mathcal{D}^{(2)}_n = \{((X_i,\widehat{V}_i,W_i),Y_i)\}_{i=1}^n$. TabCF returns a distributional estimate $\widehat{F}_{Y\mid X,V,W}(y\mid x,v,w)$ for any test point $(x,v,w)$.

\paragraph{Stage U3: estimating the interventional distributions.}
Combining the two stages, we estimate the interventional CDFs in \eqref{eq:cf_identification} by their empirical approximation defined as
\begin{equation}
\widehat{F}_{Y(x)\mid W}(y\mid w) = \frac{1}{n}\sum_{i=1}^n \widehat{F}_{Y\mid X,V,W}(y\mid x,\widehat{V}_i,w), \quad
\widehat{F}_{Y(x)}(y)
=
\frac{1}{n}\sum_{i=1}^{n}
\widehat{F}_{Y\mid X,V,W}(y\mid x,\widehat{V}_i,W_i).
\label{eq:Fhatx}
\end{equation}

We make two remarks: First, TabCF only uses TFMs for distributional regression. Due to this plug-in approach, it is convenient to obtain all intermediate nuisance estimates, which enable model diagnostics and thus improve the transparency of causal data analysis; see Appendix \ref{app:diagnostics} for further discussion. 
By contrast, it is less straightforward for recent amortized, black-box causal inference methods \citep{bynum2025BBCI,balazadeh2025causalpfn,ma2025foundation,robertson2025dopfn}.

Second, TabCF uses the full sample to estimate both stages. We also considered sample splitting and cross-fitting; however, in our synthetic experiments, cross-fitting did not improve final interventional mean or quantile accuracy and took longer runtime. We therefore use the full sample for TabCF; see Appendix \ref{app:cross-fitting}.

\section{Extending to multivariate outcomes}
\label{sec:multivariate-response}

We now extend TabCF from a univariate outcome to a $K$-dimensional response $Y=(Y_1,\ldots,Y_K)$, with $K\ge 2$.
The goal is to estimate the joint interventional distribution of $Y$ given the intervention $\DO(X=x)$:
\begin{equation}
    F_{Y(x)}(y) = \mathbb{P} \Big(Y_1\le y_1,\ldots,Y_K\le y_K \mid \DO(X=x) \Big), \qquad y=(y_1,\ldots,y_K)\in\mathbb{R}^K.
\end{equation}
By Sklar's theorem \citep{sklar1959fonctions}, there exists a copula function $C_x:[0,1]^K\to[0,1]$ such that
\begin{equation}
F_{Y(x)}(y) = C_x\Big( F_{Y_1(x)}(y_1),\ldots,F_{Y_K(x)}(y_K) \Big).
\label{eq:multivar_sklar}
\end{equation}
The marginal interventional distributions are identified under the CF assumptions. The joint interventional law additionally requires assumptions on the copula.
TabCF uses an $x$-invariant working copula $C = C_x$ to combine the marginals as a practical approximation instead of a nonparametric identification result for $F_{Y(x)}(y)$. 

The estimation proceeds in two stages:

\paragraph{Stage M1: estimating the marginal interventional distributions.} For $k=1,\ldots,K$, we estimate the marginal interventional distribution $\widehat{F}_{Y_k(x)}$ following Stages~U1--U3 in Section~\ref{sec:tabpfn}. 

\paragraph{Stage M2: fitting a copula.} We fit a copula model to summarize the dependence among the response components after the marginal interventional distributions have been estimated. The fitted copula is then used to merge the marginal estimates $\widehat{F}_{Y_k(x)}$'s into a joint interventional distribution.

\section{Synthetic experiments}
\label{sec:simulations}

We examine the operating characteristics of TabCF and compare it with the representative methods. 

\subsection{Univariate outcome: Comparative studies on interventional means}
\label{sec:sim-results-mean}

We consider the following treatment models and outcome models, where we draw $Z \sim \mathcal{N}(1.5, 0.75^2)$, and $H,\varepsilon_X,\varepsilon_Y\sim\mathcal{N}(0,1)$. These settings are adapted from \citep{holovchak2025div}.

\paragraph{Treatment models.}
We define the treatment $X$ using the following settings: 
\begin{enumerate}[label=({T\arabic*}), itemsep=3pt, topsep=2pt, parsep=0pt, partopsep=0pt]
\item \textit{Additive linear mean:} $X = Z + H + \varepsilon_X$; 
\item \textit{Quadratic mean with linear scale:} $X = \left(2Z + \tfrac{1}{4}Z^2\right) + (1 + 0.15Z)\,(H+\varepsilon_X)$.
\end{enumerate}

\paragraph{Outcome models.}
We define the outcome $Y$ using the following settings:
\begin{enumerate}[label=({O\arabic*}), itemsep=3pt, topsep=2pt, parsep=0pt, partopsep=0pt]
\item \textit{Piecewise outcome:}
$Y = \mathbf{1}_{\{X\le 1\}}\,\frac{1}{5}\bigl(5.5 + 2X + 3H + \varepsilon_Y\bigr)
   + \mathbf{1}_{\{X>1\}}\log\bigl((2X+H)^2 + \varepsilon_Y^2\bigr)$;
\item \textit{Periodic outcome with linear trend:} $Y = 3\sin(2X) + 2X - 3H + \varepsilon_Y$;  
\item \textit{Periodic nonadditive interaction:}
$Y = 1 + 2X + \cos(2X) + XH - H + \varepsilon_Y$.
\end{enumerate}
Combining the two treatment models with the three outcome models yields six settings. We do not include pretreatment covariates $W$ in the settings for comparative studies, because some competitors' implementations do not allow them. We have included additional experiments with covariates in Appendix \ref{app:pretreatment-covariates}. Note that the above treatment models satisfy the common support (\ref{CF: common support}) condition. We further examine the scenarios where (\ref{CF: common support}) is moderately violated in Appendix~\ref{app:uniform-z-common-support}.

% We consider $(Z,X,Y)\in\mathbb{R}^3$ and simulate i.i.d.\ samples from the triangular system
% \[
% X = h(Z,H,\varepsilon_X),
% \qquad
% Y = g(X,H,\varepsilon_Y),
% \]
% where the latent factor $H$ is shared across stages and induces endogeneity. 
% In the notation of the triangular model, the first-stage disturbance is $\eta=(H,\varepsilon_X)$ and the outcome disturbance is $\varepsilon=(H,\varepsilon_Y)$; hence $\eta$ and $\varepsilon$ are dependent through the shared latent factor $H$, while the instrument $Z$ is generated independently of these disturbances.
% Unless otherwise noted,  
% In the experiments, we use two treatment designs and four outcome designs:

% =========================
% 3.2 Interventional Means
% =========================

Under the above settings, we evaluate the estimation of the interventional mean function $\mu(x) = \mathbb{E}\!\left[Y \mid \DO(X=x)\right]$. 

\paragraph{Benchmark methods.}
We compare {TabCF} against a set of classical and modern baselines that are commonly used
for IV-based mean estimation. We include:
(i) {control function} (CF) estimators in linear and nonlinear variants \citep{guo2016control};
(ii) flexible ML-based IV methods such as DIV \citep{holovchak2025div}, DeepIV \citep{hartford2017deepiv} and DeepGMM \citep{bennett2019deepgmm};
and (iii) {TabPFN-naive} \cite{hollmann2025tabpfn}, which directly estimates the conditional distribution as if there was no confounding (without IV adjustment).
Appendix~\ref{app:benchmarks} provides short descriptions
and implementation details for each baseline. 

\begin{figure}[ht]
    \centering
\includegraphics[width=0.85\linewidth]{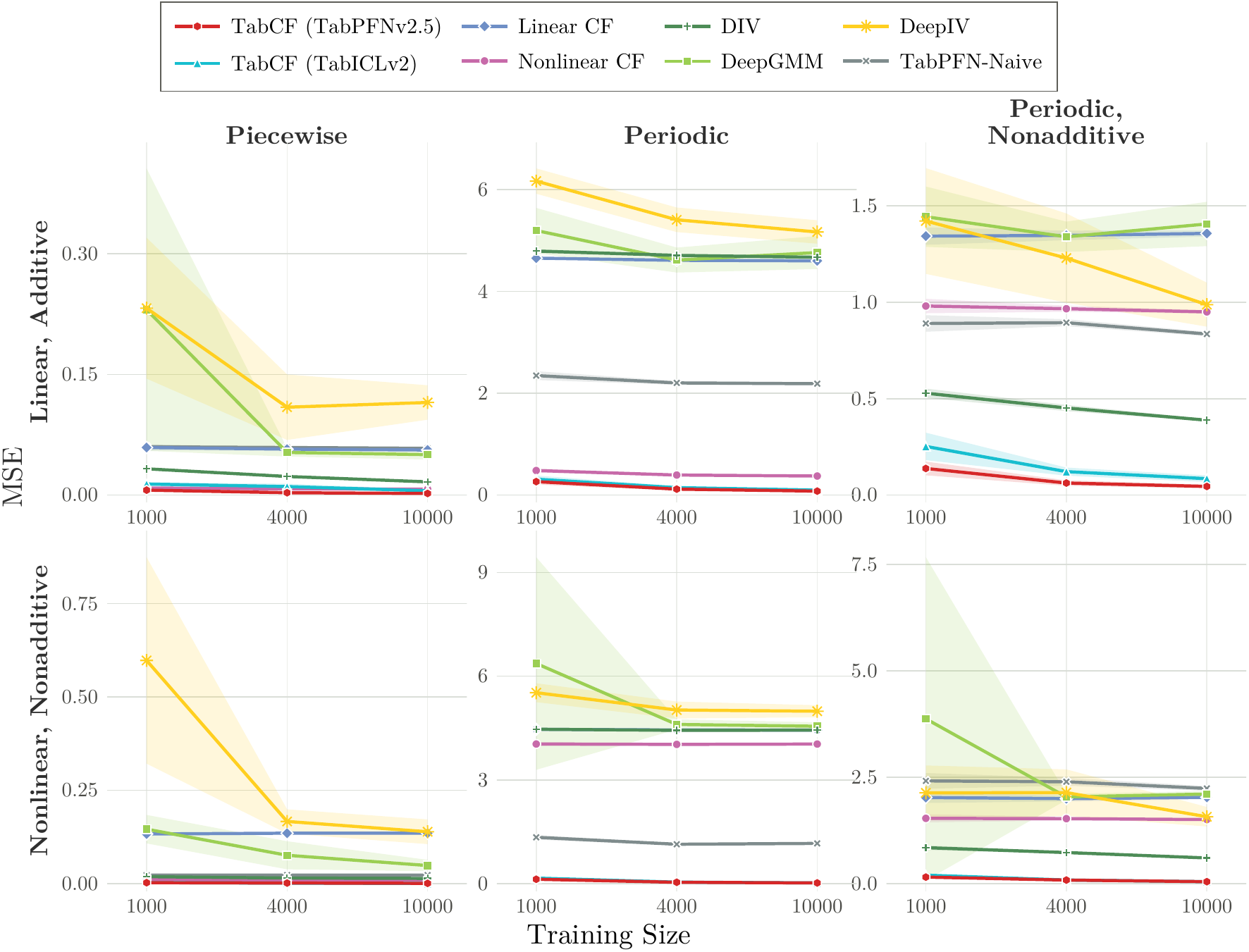}
    \caption{Results of interventional mean estimation. Rows correspond to treatment settings; columns correspond to outcome settings. Results are averaged over 100 random seeds with standard deviations in shaded regions.}
    \label{fig:sim_mean_a3_grid}
\end{figure}

% NOTE: If your bib keys differ, update the citation keys accordingly.
% TODO: ensure the bibliography contains the cited baseline references used in Appendix~\ref{app:benchmarks}.

\paragraph{Metric.}
We report the mean squared error (MSE) of the estimated interventional mean curve on a held-out grid of intervention levels.
Let the evaluation grid \(\{x_g\}_{g=1}^{200}\) be equally spaced from the lower 5\% quantile to the upper 5\% quantile of $X$ distribution.
For each $x_g$, we approximate the ground-truth mean
$\mu(x_g)$ by Monte Carlo under $\DO(X=x_g)$: draw $5000$ i.i.d.\ samples $\{y_{gm}\}_{m=1}^M$ from the true interventional
distribution and compute $\bar y_g = \frac{1}{5000}\sum_{m=1}^{5000} y_{gm}$.
For an estimator $\widehat{\mu}(x)$, we report $\operatorname{MSE} = \frac{1}{200} \sum_{g=1}^
{200}\bigl(\widehat{\mu}(x_g)-\bar y_g\bigr)^2$.
% TODO: state (i) the chosen M, (ii) the number of repetitions/seeds, and (iii) whether parentheses in tables indicate SD/SE across seeds.

% -------------------------
% Results tables/figures
% -------------------------

\paragraph{Results.}
Figure~\ref{fig:sim_mean_a3_grid} shows that the hardest estimation settings are the periodic and
periodic-nonadditive outcome designs, where endogeneity-ignorant or less tailored baselines can incur
substantially larger errors. Across different scenarios, {TabCF} remains in the best-performing cluster and is especially stable as the outcome mechanism becomes more nonlinear. %\textcolor{red}{We also find some ... a larger sample size does not lead to better performance for TabCF...; see Section ?}

% ============================
% 3.3 Interventional Quantiles
% ============================
\subsection{Univariate outcome: Comparative studies on interventional quantiles}
\label{sec:sim-results-quantile}

For the same settings in Section \ref{sec:sim-results-mean}, we evaluate the estimation of the interventional quantile function
$q_\tau(x) = \inf\{y\in\mathbb{R}: F_{Y(x)}(y) \ge \tau\}$ for $\tau\in(0,1)$. 

\paragraph{Benchmark methods.}
We compare {TabCF} against two quantile baselines: DIV \citep{holovchak2025div} and IVQR \citep{chernozhukov2005ivqr}.
For {TabCF} and DIV, interventional quantiles are obtained by inverting the estimated interventional CDF,
whereas IVQR estimates a quantile-specific structural curve directly at each specified level.
Appendix~\ref{app:benchmarks} summarizes the benchmark methods and their scope.

\paragraph{Metric.}
Again, let the evaluation grid \(\{x_g\}_{g=1}^{200}\) be equally spaced from the lower 5\% quantile to the upper 5\% quantile of the distribution of $X$.
Let $q_\tau(x_g)$ denote a Monte Carlo approximation to the ground truth
interventional $\tau$-quantile at $x_g$.
For each $\tau$, we report pointwise errors aggregated over the grid, 
{$\mathrm{MSE}(\tau) = \frac{1}{200}\sum_{g=1}^{200}(\widehat q_\tau(x_g)-q_\tau(x_g))^2$}.

\newcommand{\SimQuantileRMSEBoxplots}{%
  \begin{figure}[!htbp]
    \centering
    \includegraphics[width=0.8\linewidth]{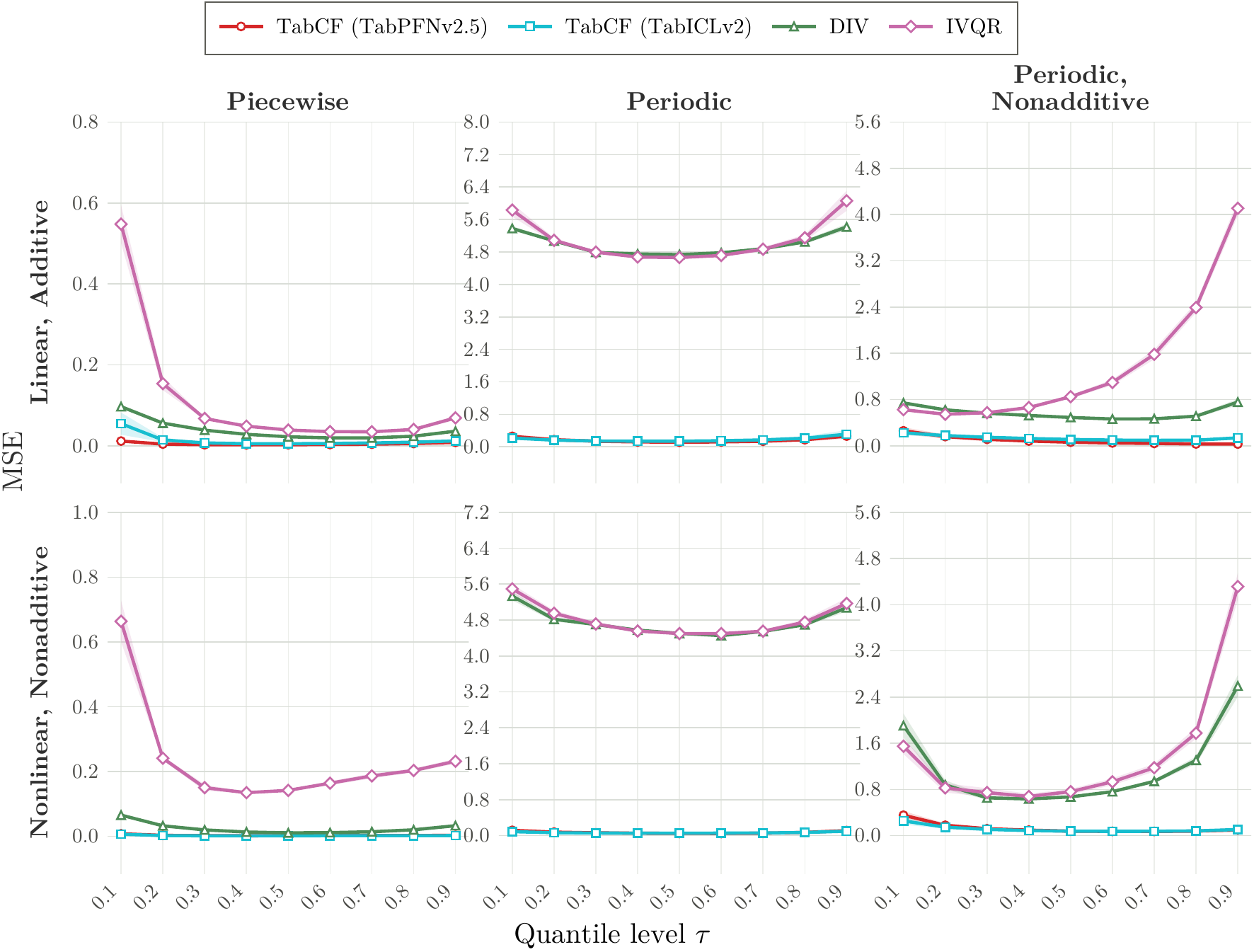}
    \caption{Interventional quantile MSE at $n=4000$. Rows correspond to treatment settings; columns correspond to outcome settings. Results are averaged over 100 random seeds with standard deviations in shaded regions.}
    \label{fig:sim-quantile-rmse-boxplots}
  \end{figure}%
}

\SimQuantileRMSEBoxplots

\paragraph{Results.}
Figure~\ref{fig:sim-quantile-rmse-boxplots} shows a clear separation between the TabCF and the existing competitors.
Across all six settings, both {TabCF} variants substantially outperform DIV and IVQR, often by a wide margin. Within the foundation model family, the TabPFN backbone performs best on linear-additive treatment cases with piecewise and periodic-nonadditive outcomes, while the TabICL backbone performs best on the remaining four designs, including nonlinear treatment cases. The strongest gains appear in the periodic and periodic-nonadditive outcome settings, where the gap between TabCF and DIV/IVQR is largest.
Errors remain smallest over the middle quantiles and rise as $\tau$ approaches $0.1$ or $0.9$, but this edge degradation is still milder for the two {TabCF} variants than for the classical baselines.
Overall, TabCF is most useful when the structural relationship is indeed nonlinear.

\subsection{Univariate outcome: Runtime analysis}
\label{app:runtime}
\begin{figure}[!htbp]
\centering
\includegraphics[width=0.6\linewidth]{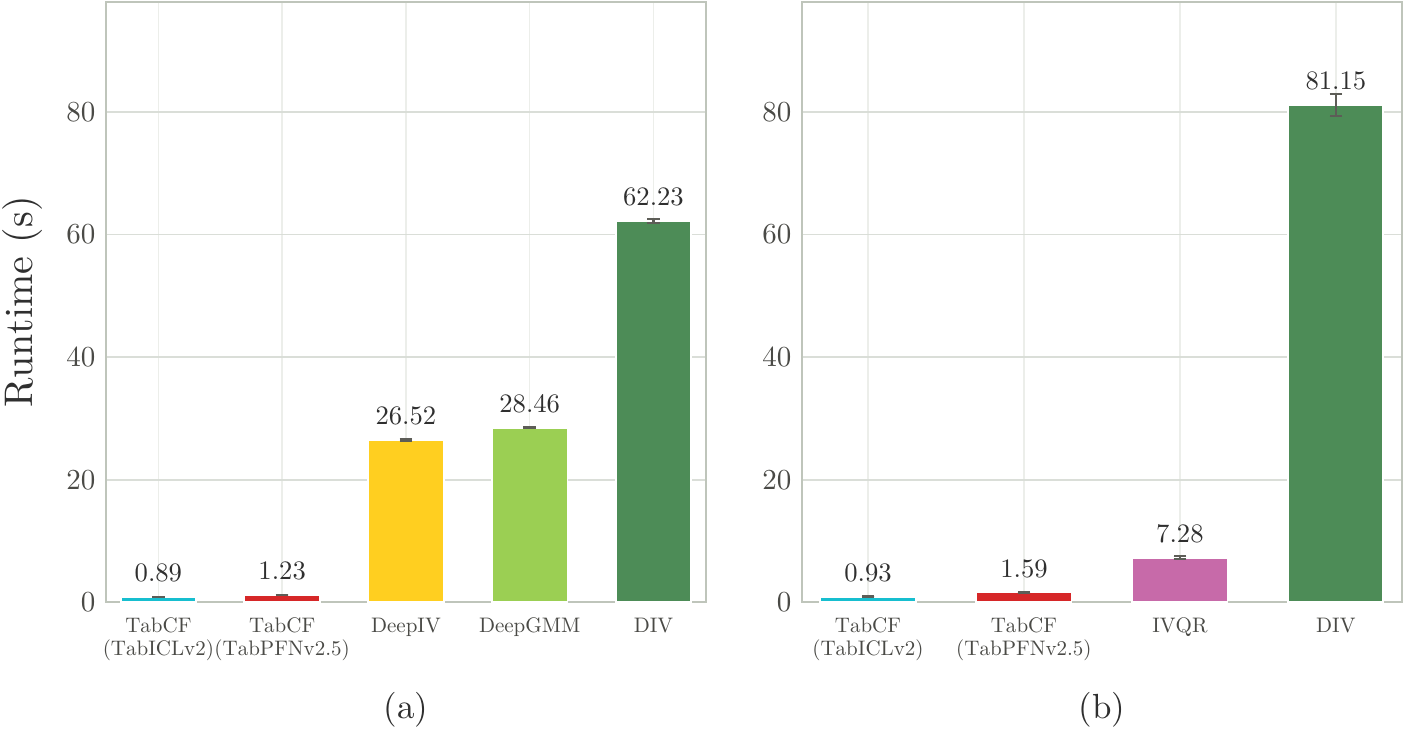}
\caption{Runtime comparison at $n=1000$: (a) reports the results of the interventional mean estimation, and (b) reports the results of the
interventional quantile estimation. Results are averaged over 100 random seeds with error bars.}
\label{fig:runtime-analysis-panels}
\end{figure}

Beyond estimation accuracy, computational efficiency is also important for practical IV estimation, especially when competing methods require repeated model fitting or distributional optimization. Figure~\ref{fig:runtime-analysis-panels} summarizes the runtime comparison, among nonlinear methods, for estimating interventional means and quantiles at $n=1000$. TabCF is significantly faster than the state-of-the-art methods. 

In our experiments, all GPU experiments were run on a cluster with NVIDIA A100-SXM4-80GB GPUs, using PyTorch 2.8.0+cu128; jobs were submitted with 8 CPU cores and 64GB of host memory. The R baselines ({Linear CF}, {Nonlinear IV}, and {DIV}) are launched through the cluster's R module stack as CPU jobs with eight CPU cores and 64GB RAM, while {TabCF}, {TabPFN-naive}, {DeepIV}, and {DeepGMM} use Python implementations.

\subsection{Multivariate outcomes: Comparative studies on joint interventional distribution}
\label{sec:sim-results-multivariate}

We evaluate the extension of TabCF in the bivariate case $Y=(Y_1,Y_2)$. Across all designs, we use the setup:
$Z \sim \mathcal{N}(1.5, 0.75^2)$, $H \sim \mathcal{N}(0,1)$, and $\varepsilon_X \sim \mathcal{N}(0,1)$
and $\varepsilon_1,\varepsilon_2 \sim \mathcal N(0, 1)$ with correlation $\rho_{\varepsilon}$, with all latent variables mutually independent except for the correlation in
$(\varepsilon_1,\varepsilon_2)$.
Note that $\rho_\varepsilon$ controls the correlation of the outcome noise only.
The treatment is generated as $X = Z + H + \varepsilon_X$. 
We consider four bivariate outcome settings:
\begin{enumerate}[label=({BO\arabic*}), itemsep=3pt, topsep=2pt, parsep=0pt, partopsep=0pt]
 \item \textit{Linear baseline:} Let
$(Y_1, Y_2)^\top = (1, 0.5)^\top X 
- (3, 1)^\top H
+
(\varepsilon_1, \varepsilon_2)^\top$.
% Under $do(X=x)$, the joint law is bivariate Gaussian and therefore available
% analytically.
\item \textit{Nonlinear outcomes:} Let $Y_1 = \psi_1(X+H+\varepsilon_1)$ and $Y_2 = \psi_2(X+H+\varepsilon_2)$, where $\psi_1(w) = 2w + 3\sin(2w)$ and 
$\psi_2(w) = \tfrac{1}{2}\psi_1(w)$. 
This yields non-Gaussian interventional outcome distributions through a nonlinear transformation.
\item \textit{Piecewise pre-additive outcomes:}
Let $Y_1 = h_1(X+H+\varepsilon_1)$ and $Y_2 = h_2(X+H+\varepsilon_2)$, where $h_1(w) = \mathbf{1}_{\{w<0\}} w
+ \mathbf{1}_{\{0\le w\le 1\}} 2w
+ \mathbf{1}_{\{w>1\}}(2+\tfrac{1}{2}(w-1))$ and $h_2(w) = \tfrac{1}{2} h_1(w)$.
This introduces kinks at $w=0$ and $w=1$ to test robustness to non-smooth pre-additive transformations.
\item \textit{Softplus pre-additive outcomes:}
Let $Y_1 = s_1(X+H+\varepsilon_1)$ and $Y_2 = s_2(X+H+\varepsilon_2)$, where $s_1(w) = \log\bigl(1+\exp(2w)\bigr)$ and $s_2(w) = \tfrac{1}{2} s_1(w)$.
This design replaces the piecewise map with a smooth monotone alternative,
yielding non-Gaussian marginals without derivative discontinuities.
\end{enumerate}

\paragraph{Benchmark methods.} We compare four methods that accommodate multivariate outcomes:
(i) TabCF; (ii) DIV \citep{holovchak2025div}; (iii) TabPFN-naive \citep{hollmann2025tabpfn}, which ignores the
instrument in the marginal stage; and (iv) an independence baseline, which estimates the same interventional marginal distributions but ignores the dependence structure of the joint distribution.

\paragraph{Metric.}
For evaluation, we compute the sliced Wasserstein distance
between the estimated and true joint interventional distributions on an equally spaced grid $\{x_g\}_{g=1}^{200}$ over $[0,3]$. 

\begin{figure}[!htbp]
    \centering
    \includegraphics[width=\textwidth]{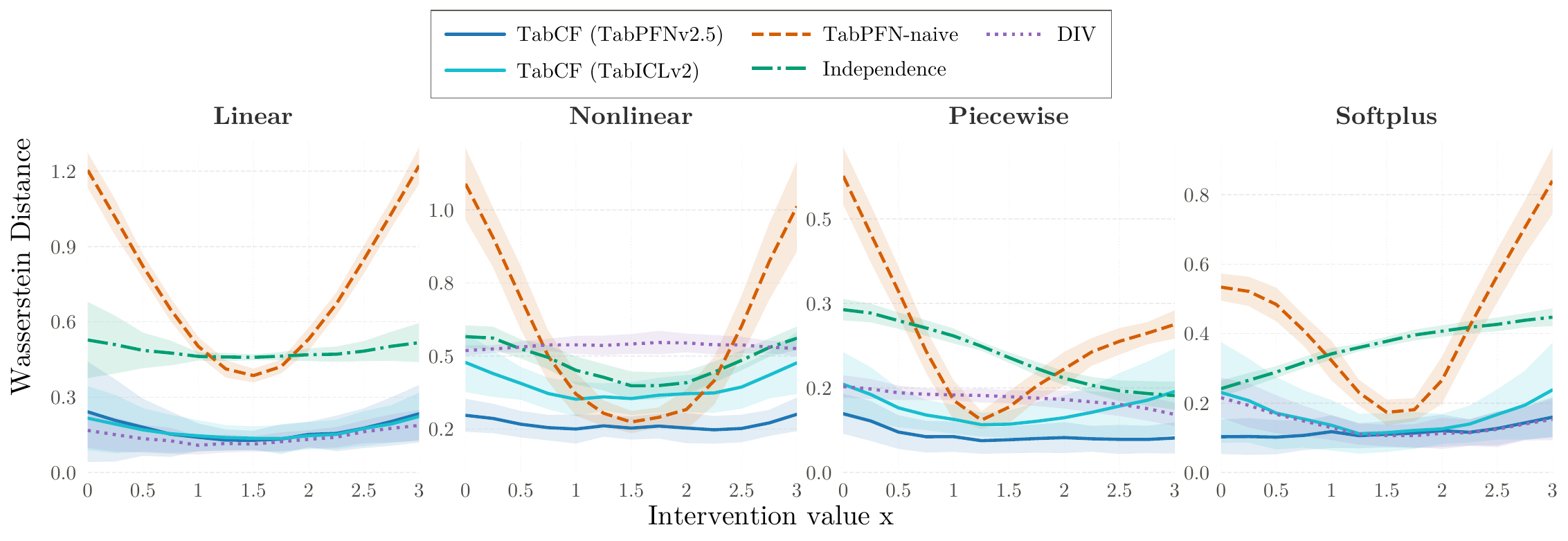}
    \caption{Bivariate-outcome sliced Wasserstein distance to the oracle joint
    interventional distribution for four settings when $n=2000$ and $\rho_{\varepsilon} = 0.6$. Results are averaged over 100 random seeds with standard deviations in shaded regions. Lower
    values indicate better results.}
    \label{fig:multivar_official}
\end{figure}

\paragraph{Results.}
Figure~\ref{fig:multivar_official} shows the results of bivariate-outcome
comparisons. Across various scenarios, TabCF consistently ranks among the strongest methods.

\subsection{Additional experiments}

Appendix \ref{app:experiments} contains additional experiments, including settings with pretreatment covariates in Appendix \ref{app:pretreatment-covariates}, violations of common support (\ref{CF: common support}) in Appendix \ref{app:uniform-z-common-support}, and weak IV relevance (\ref{iv: relevance}) in Appendix \ref{app:instrument-relevance}. Overall, TabCF continues exhibiting strong performance in these settings.

\section{Real data examples}
\label{sec:real-data}

We complement the simulation study with four real data examples:
\textbf{(AJR)} the colonial-origins design for institutions and long-run development \citep{acemoglu2001colonial};
\textbf{(Fulton Fish)} demand at the Fulton Fish Market using weather-based instruments \citep{graddy1995fish};
\textbf{(Card)} the Card college-proximity design for returns to schooling \citep{card1995geographic}; and
\textbf{(CigarettesSW)} state-level cigarette demand \citep{stockwatson2007introduction}. 
Following the evaluations used in recent
distributional-IV work~\citep{holovchak2025div}, we report interventional mean curves for all
four applications and, for \textbf{Fulton Fish}, an additional compact interventional quantile
comparison. Further details on the empirical settings, including the treatment, instrument, outcome, and interpretation of each benchmark, are provided in Appendix~\ref{app:real-data-details}.

% The purpose of these applications is not to provide a definitive empirical re-analysis of each original study. Instead, we use them as standardized scalar-IV benchmarks for assessing whether TabCF produces economically interpretable and appropriately regularized causal curves on small or moderate-sized tabular data. For comparability across applications, we retain only one excluded instrument, one endogenous treatment, and one outcome in each data set. Continuous-treatment applications are evaluated on a common grid over the observed treatment support, while the schooling application is evaluated only at observed schooling levels because education is ordered and discrete.

\begin{figure}[!htbp]
\centering
\includegraphics[width=1\textwidth]{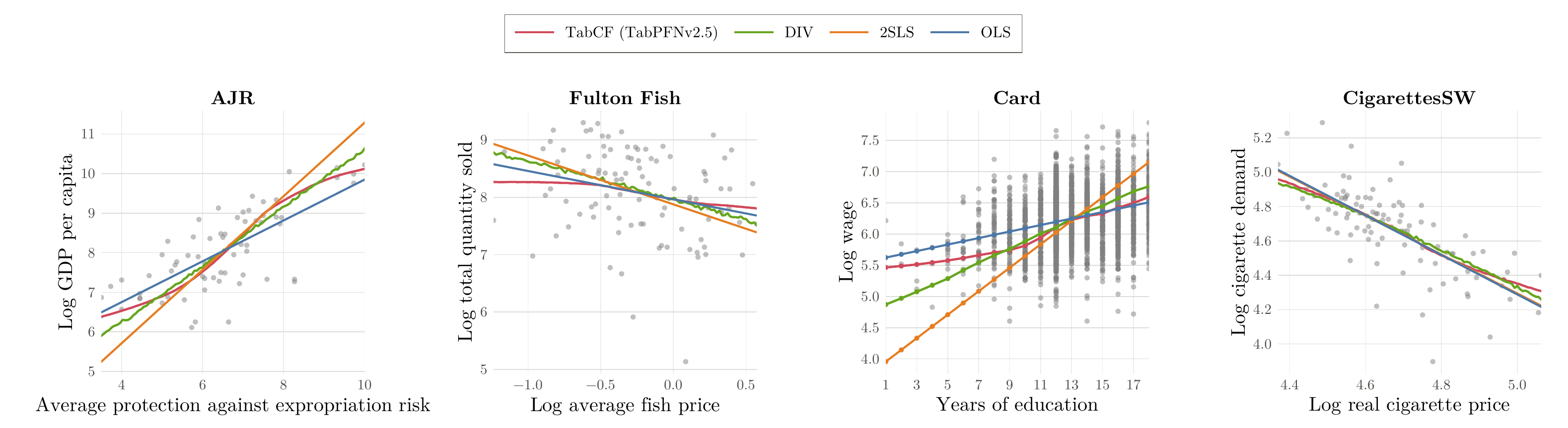}
\caption{Results of interventional mean estimations on {\bf AJR}, {\bf Fulton Fish}, {\bf Card}, and {\bf CigarettesSW}.}
\label{fig:real-data-official}
\end{figure}

\begin{figure}[!htbp]
\centering
\includegraphics[width=0.8\textwidth]{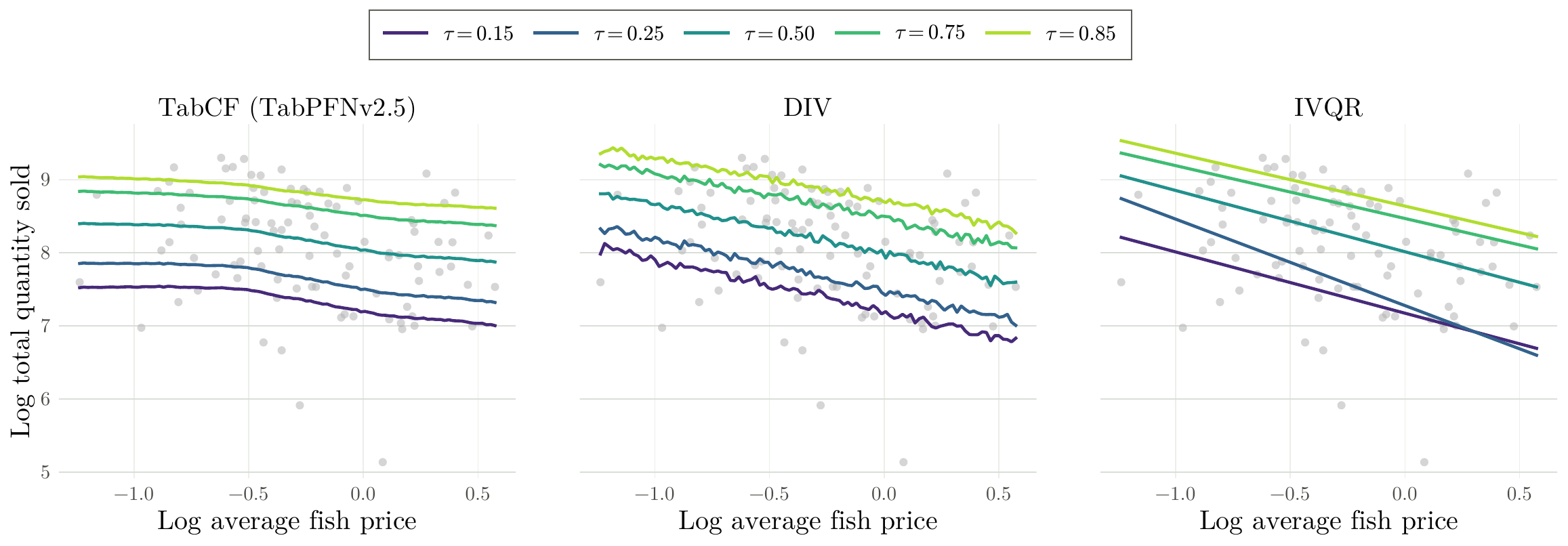}
\caption{Results of interventional quantile estimation on \textbf{Fulton Fish}.}
\label{fig:real-data-fulton-quantile}
\end{figure} 

Figures~\ref{fig:real-data-official}--\ref{fig:real-data-fulton-quantile} summarize the results. In \textbf{AJR} and \textbf{Card}, TabCF produces curves qualitatively consistent with standard economic interpretations. In \textbf{Fulton Fish} and \textbf{CigarettesSW}, TabCF yields plausible downward-sloping demand responses. Overall, these examples illustrate that TabCF produces economically interpretable curves under standard benchmark IV settings.

%Across all panels, the goal is structural interpretability and agreement with established IV benchmarks, not supervised prediction against an observed counterfactual target.

\section{Discussion}
\label{sec:discussion}

In this work, we introduce TabCF, a simple yet effective control function method based on tabular foundation models. TabCF offers accurate, fast, transparent, and tuning-light causal estimation of distributional quantities. 
Moreover, it is intended as a general control function framework rather than as an estimator tied to a particular tabular foundation model. The foundation model is used only as a probabilistic regression module. This modularity is useful because the same causal wrapper can, in principle, benefit from future tabular foundation models with better calibration, longer contexts, or faster inference, without changing the underlying causal identification argument.

\paragraph{Limitations.}
Despite the promising performance, TabCF has several limitations. It inherits the weaknesses of the current tabular foundation models. For example, it heavily relies on GPU acceleration, while CPU inference may be
substantially slower. Moreover, in contrast to conventional methods, TabCF performance does not necessarily improve with an increasing sample size (as shown in Appendix Figure \ref{fig:app-uniform-z-mean}), due to the context window limit of foundation models \citep{nagler2023pfn,ma2025tabdpt}. Subsampling strategies may be useful for addressing this issue \citep{thomas2024localpfn}. Finally, the IV and CF methods rely on inherently untestable assumptions. In this work, we did not consider the possibility of invalid instruments \citep{chen2024discovery}. This can be partially addressed by using median estimators \citep{hartford2021valid}. We leave this for future research.

\section*{Acknowledgment}

This research is supported by U.S. NSF grant DMS-2515789.

\bibliographystyle{plainnat}
\bibliography{tmlr}
%reuse the bibliography from TMLR version.

\clearpage

\appendix

\section{Further considerations on TabCF design}

\subsection{Comparison of explicit and implicit control function approaches}
\label{app:explicit-implicit-cf}

Here, we further discuss the explicit and implicit control function (CF) approaches in
Section~\ref{sec:identification}. We record one of our early (failed) attempts to integrate tabular foundation models (TFMs) with an ``implicit'' identification approach \citep{lee2007endogeneity}. 
The following experiment compares estimators using an explicit (TabCF) or an implicit identification \citep{lee2007endogeneity}. We evaluate them with the interventional quantile setting in
Section~\ref{sec:sim-results-quantile}. TabPFNv2.5 is used for both estimators in this comparison.

\begin{figure}[!htbp]
\centering
\includegraphics[width=.8\textwidth]{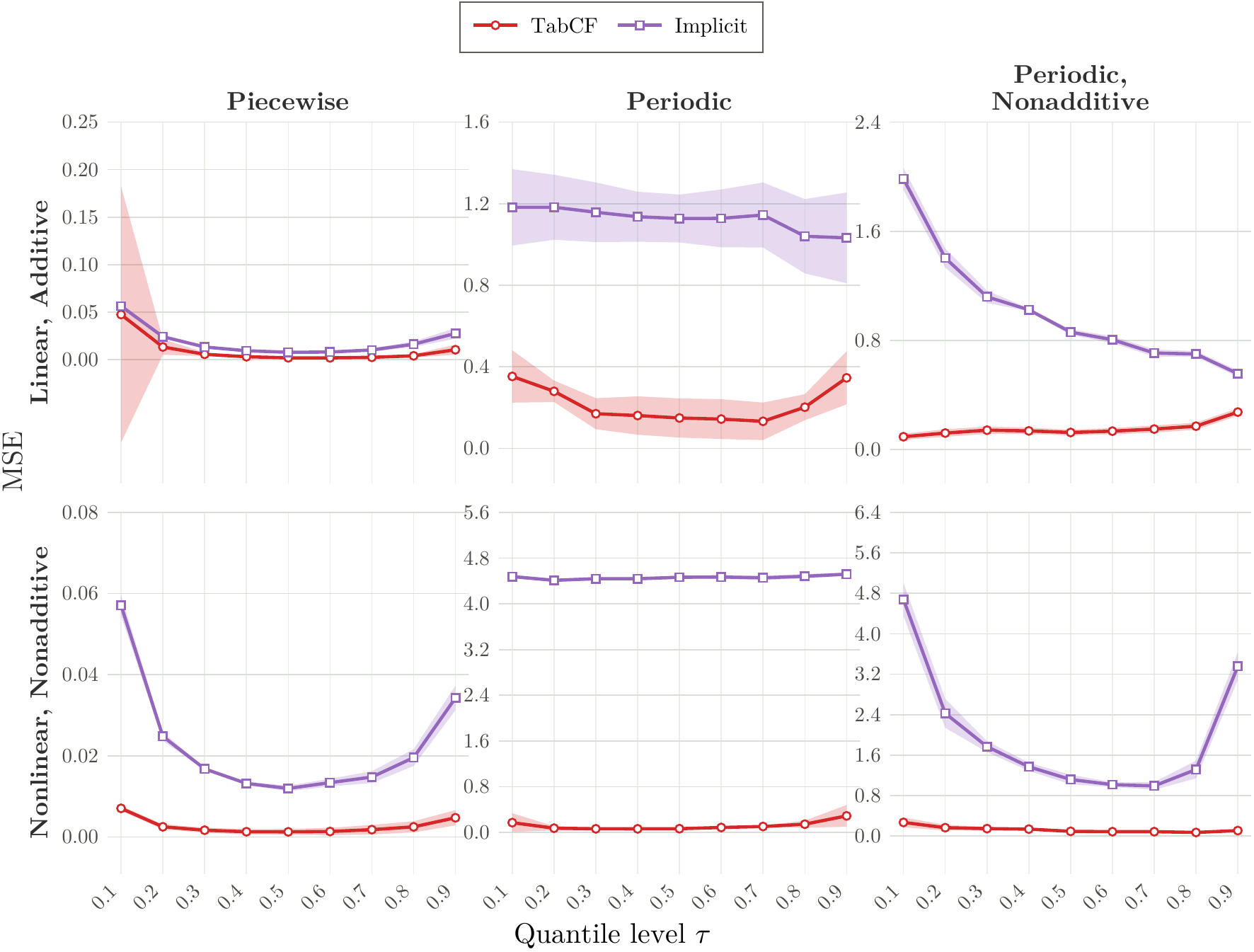}
\caption{Comparison between TabCF and an implicit control function approach \citep{lee2007endogeneity} on interventional quantile estimation. Both methods use TabPFNv2.5 for nuisance estimation. Results are averaged over 30 random seeds with standard deviations in shaded regions.}
\label{fig:explicit-implicit-cf-diagnostic}
\end{figure}

Figure~\ref{fig:explicit-implicit-cf-diagnostic} shows that integrating the implicit CF approach with TFM has larger estimation errors. This is likely due to error amplification when solving equations, performing optimization, or solving inverse problems.

\subsection{Diagnostics on control function}
\label{app:diagnostics}

In TabCF, the most important intermediate quantity is $\widehat{V}$. We recommend the following empirical checking to assess its quality.

\paragraph{Goodness-of-fit checking.} A good fitting should yield 
\begin{equation*}
    \widehat V\mid (W=w)\sim \mathrm{Unif}(0,1) \quad 
    \text{and} \quad
    \widehat{V}\independent Z\mid (W = w),
\end{equation*}
approximately, for any relevant $w$. When $W$ is absent, one can use Q-Q plot to assess $\widehat V \sim \mathrm{Unif}(0,1)$ and nonparametric testing (e.g., distance correlation test \citep{szekely2007measuring}) to assess $\widehat{V}\independent Z$. When $W$ is of low dimensions, one can stratify $W$ and check each stratum. If $W$ is of high dimensions, one should resort to other approaches, which is beyond the scope of this work.
It is worth noting that passing this goodness-of-fit checking does not imply the IV and CF conditions are satisfied.

\paragraph{Conditional independence checking.}
If the IV and CF conditions are well satisfied, then one should have 
\begin{equation*}
    Y\independent Z\mid (X = x, \widehat{V} = v, W = w)
\end{equation*}
approximately, for any relevant $(x, v, w)$. This can be examined by nonparametric testing or independence scores (e.g., FOCI \citep{azadkia2021simple}). Failure of this diagnostic indicates a violation of some parts of the IV/CF specification, but passing it does not prove the conditions are met.

\subsection{Comparison of full-sample and cross-fitted control function}
\label{app:cross-fitting}

As described in Section \ref{sec:tabpfn}, TabCF uses the full sample for CF construction. An alternative is to use sample splitting and cross-fitting \citep{chernozhukov2018double} to reduce the potential impact of overfitting nuisances. Here, we compare the full-sample and cross-fitting approaches for TabCF.

The TabCF estimator fits the first-stage conditional CDF on the full training sample and evaluates $\widehat V_i=\widehat F_{X\mid Z}(X_i\mid Z_i)$ on the same observations, whereas cross-fitting uses five folds: for each fold, the first-stage model is fitted on the other four folds, and $\widehat V_i$ is predicted out of fold.
The same second-stage TabCF estimator is then fitted using these $\widehat V_i$ values.

For interventional means, the MSE is computed on the same intervention grids as in Section~\ref{sec:sim-results-mean},
\[
    \frac{1}{200}\sum_{g=1}^{200}\{\widehat\mu(x_g)-\mu(x_g)\}^2,
\]
with the true mean approximated by a Monte Carlo simulation.
For interventional quantiles, let $\mathcal T=\{0.1,0.2,\ldots,0.9\}$.
We first compute the grid MSE at each quantile level,
\[
    \mathrm{MSE}(\tau)=
    \frac{1}{200}\sum_{g=1}^{200}\{\widehat q_\tau(x_g)-q_\tau(x_g)\}^2,
    \qquad \tau\in\mathcal T,
\]
and then report the average $\frac{1}{9}\sum_{\tau\in\mathcal T}\mathrm{MSE}(\tau)$.
We set the sample size $n=1000$.

\begin{table}[htbp]
    \centering
    \caption{Estimation comparison of full-sample and cross-fitted control function. Results are averaged over 30 random seeds with standard deviations in parentheses. The quantile MSE column averages the nine $\tau$-specific MSEs for $\tau=0.1,\ldots,0.9$.}
    \small
    \label{tab:fullsample_crossfit}
    \begin{tabular}{llcc}
        \toprule
        \textbf{Setting} & \textbf{Approach} & \textbf{MSE for Mean} & \textbf{MSE for Quantile}\\
        \midrule
        (T1,O1) & Full-sample $\hat{V}$ & 0.00 (0.00) & 0.00 (0.00) \\
        (T1,O1) & 5-fold cross-fitted $\hat{V}$ & 0.00 (0.00) & 0.00 (0.00) \\
        \addlinespace
        (T1,O2) & Full-sample $\hat{V}$ & 0.15 (0.14) & 0.17 (0.08) \\
        (T1,O2) & 5-fold cross-fitted $\hat{V}$ & 0.15 (0.14) & 0.17 (0.08) \\
        \addlinespace
        (T1,O3) & Full-sample $\hat{V}$ & 0.08 (0.08) & 0.09 (0.05) \\
        (T1,O3) & 5-fold cross-fitted $\hat{V}$ & 0.08 (0.08) & 0.09 (0.05) \\
        \midrule
        (T2,O1) & Full-sample $\hat{V}$ & 0.00 (0.00) & 0.00 (0.00) \\
        (T2,O1) & 5-fold cross-fitted $\hat{V}$ & 0.00 (0.00) & 0.00 (0.00) \\
        \addlinespace
        (T2,O2) & Full-sample $\hat{V}$ & 0.06 (0.06) & 0.07 (0.02) \\
        (T2,O2) & 5-fold cross-fitted $\hat{V}$ & 0.06 (0.06) & 0.07 (0.02) \\
        \addlinespace
        (T2,O3) & Full-sample $\hat{V}$ & 0.10 (0.07) & 0.13 (0.05) \\
        (T2,O3) & 5-fold cross-fitted $\hat{V}$ & 0.10 (0.07) & 0.13 (0.05) \\
        \bottomrule
    \end{tabular}
\end{table}

\begin{table}[htbp]
    \centering
    \caption{Runtime comparison of full-sample and cross-fitted control function. Results are averaged over 30 random seeds with standard deviations in parentheses.}
    \small
    \label{tab:fullsample_crossfit_runtime}
    \begin{tabular}{llccc}
        \toprule
        \textbf{Setting} & \textbf{Approach} & \textbf{Stage U1} & \textbf{Stage U2/U3} & \textbf{Total Runtime} \\
        \midrule
        (Mean, TabPFN) & Full-sample $\hat{V}$ & 0.71 (0.00) & 0.52 (0.00) & 1.23 (0.00) \\
        (Mean, TabPFN) & 5-fold cross-fitted $\hat{V}$ & 3.28 (0.06) & 0.50 (0.00) & 3.78 (0.06) \\
        \addlinespace
        (Mean, TabICL) & Full-sample $\hat{V}$ & 0.42 (0.00) & 0.48 (0.02) & 0.89 (0.02) \\
        (Mean, TabICL) & 5-fold cross-fitted $\hat{V}$ & 1.92 (0.06) & 0.48 (0.04) & 2.40 (0.07) \\
        \midrule
        (Quantile, TabPFN) & Full-sample $\hat{V}$ & 0.67 (0.00) & 0.54 (0.00) & 1.21 (0.01) \\
        (Quantile, TabPFN) & 5-fold cross-fitted $\hat{V}$ & 3.17 (0.01) & 0.54 (0.04) & 3.72 (0.04) \\
        \addlinespace
        (Quantile, TabICL) & Full-sample $\hat{V}$ & 0.37 (0.02) & 0.56 (0.03) & 0.93 (0.03) \\
        (Quantile, TabICL) & 5-fold cross-fitted $\hat{V}$ & 1.78 (0.05) & 0.57 (0.00) & 2.35 (0.05) \\
        \bottomrule
    \end{tabular}
\end{table}

Table~\ref{tab:fullsample_crossfit} shows that cross-fitting leaves the final mean and quantile errors essentially unchanged across the six settings.
Table~\ref{tab:fullsample_crossfit_runtime} shows that the extra cost is concentrated in Stage U1, where cross-fitting requires five first-stage fits, while Stage U2/U3 remains nearly unchanged.
We therefore use the full-sample control function for TabCF implementation.

\section{Proof of Proposition \ref{proposition:identification}}
\label{app:proof_cf_identification}

\begin{proof}
The proof proceeds in three steps.

\textbf{Step 1: Conditional independence.}
Fix $w$ in the support of $W$ and $z$ in the conditional support of $Z\mid W=w$.
By (\ref{CF: monotone}), without loss of generality, we assume the map $\eta\mapsto h(z,w,\eta)$ is strictly increasing and continuous,
hence invertible in its third argument; denote the inverse by $h^{-1}(z,w,\cdot)$.
Using $X=h(Z,W,\eta)$ and (\ref{iv: exogeneous}) unconfoundedness, for any $x$ we have
\[
\begin{aligned}
F_{X\mid Z,W}(x\mid z,w)
&=
\mathbb{P}\{h(z,w,\eta)\le x\mid Z=z,W=w\}  \\
&=
\mathbb{P}\{\eta\le h^{-1}(z,w,x)\mid Z=z,W=w\}  \\
&=
\mathbb{P}\{\eta\le h^{-1}(z,w,x)\mid W=w\}  \\
&=
F_{\eta\mid W}\!\left(h^{-1}(z,w,x)\mid w\right).
\end{aligned}
\]
Evaluating at $x=X=h(Z,W,\eta)$ gives
\[
V
=
F_{X\mid Z,W}(X\mid Z,W)
=
F_{\eta\mid W}(\eta\mid W)
\qquad \text{a.s.}
\]
Because $F_{\eta\mid W}(\cdot\mid w)$ is continuous and strictly increasing,
the probability integral transform implies
\[
V\mid W=w\sim \mathrm{Unif}(0,1).
\]
Moreover, since $V$ is a measurable function of $(\eta,W)$ and
$Z\independent \eta\mid W$, we have
\[
V\independent Z\mid W.
\]

\textbf{Step 2: Control function exogeneity.}
Since $V=F_{\eta\mid W}(\eta\mid W)$ and $F_{\eta\mid W}(\cdot\mid W)$ is strictly
increasing, conditioning on $(V,W)$ is equivalent to conditioning on $(\eta,W)$.
Under (\ref{iv: exogeneous}), $Z\independent (\eta,\varepsilon)\mid W$, and hence
$Z\independent \varepsilon\mid \eta,W$.
Because $X=h(Z,W,\eta)$ is measurable with respect to $(Z,W,\eta)$, it follows that
\[
\varepsilon\independent X\mid \eta,W,
\qquad\text{and therefore}\qquad
\varepsilon\independent X\mid V,W.
\]
Consequently, for all $(x,v,w)$ in the relevant support,
\[
\begin{aligned}
F_{Y\mid X,V,W}(y\mid x,v,w)
&=
\mathbb{P}(Y\le y\mid X=x,V=v,W=w)  \\
&=
\mathbb{P}\{g(X,W,\varepsilon)\le y\mid X=x,V=v,W=w\}  \\
&=
\mathbb{P}\{g(x,w,\varepsilon)\le y\mid X=x,V=v,W=w\}  \\
&=
\mathbb{P}\{g(x,w,\varepsilon)\le y\mid V=v,W=w\}.
\end{aligned}
\]

\textbf{Step 3: Interventional CDF.}
Under the intervention $\DO(X=x)$, the structural outcome equation becomes
$Y=g(x,W,\varepsilon)$. For fixed $w$,
\[
\begin{aligned}
F_{Y(x)\mid W=w}(y)
&=
\mathbb{P}\{g(x,w,\varepsilon)\le y\mid W=w\} \\
&=
\int
\mathbb{P}\{g(x,w,\varepsilon)\le y\mid V=v,W=w\}
\,dF_{V\mid W=w}(v) \\
&=
\int_0^1
F_{Y\mid X,V,W}(y\mid x,v,w)\,dv,
\end{aligned}
\]
where the last equality uses $V\mid W=w\sim \mathrm{Unif}(0,1)$.
Finally, since $W$ is pretreatment,
\[
\begin{aligned}
F_{Y(x)}(y)
&=
\mathbb{E}_{W}\!\left[F_{Y(x)\mid W}(y\mid W)\right] \\
&=
\mathbb{E}_{W}\!\left[
\int_0^1
F_{Y\mid X,V,W}(y\mid x,v,W)\,dv
\right].
\end{aligned}
\]
Assumption (\ref{CF: common support}) ensures that
$F_{Y\mid X,V,W}(y\mid x,v,w)$ is well-defined for all relevant $v$ and for
$F_W$-almost every $w$. This proves the identification result.
\end{proof}

\section{Additional synthetic experiments}
\label{app:experiments}

\subsection{Experiments with pretreatment covariates}
\label{app:pretreatment-covariates}

This section examines the settings in which observed pretreatment covariates are
available. We repeat the interventional quantile estimation experiments in
Sections~\ref{sec:sim-results-quantile} after a
covariate augmentation of the same six treatment-outcome settings. Let $h_0$ and $g_0$ denote the treatment and outcome structural functions in Section \ref{sec:sim-results-mean}, respectively. We draw
an observed pretreatment covariate $W\sim\mathcal N(0,1)$ independently, and set
\[
    X=h_0(Z,H,\varepsilon_X)+0.5W,
    \qquad
    Y=g_0(X,H,\varepsilon_Y)+0.5W.
\]

% Methods that support covariates condition on $W$ in the relevant stages, and the reported
% interventional quantiles are marginalized over the empirical distribution of $W$.

\begin{figure}[!htbp]
    \centering
    \includegraphics[width=0.8\textwidth]{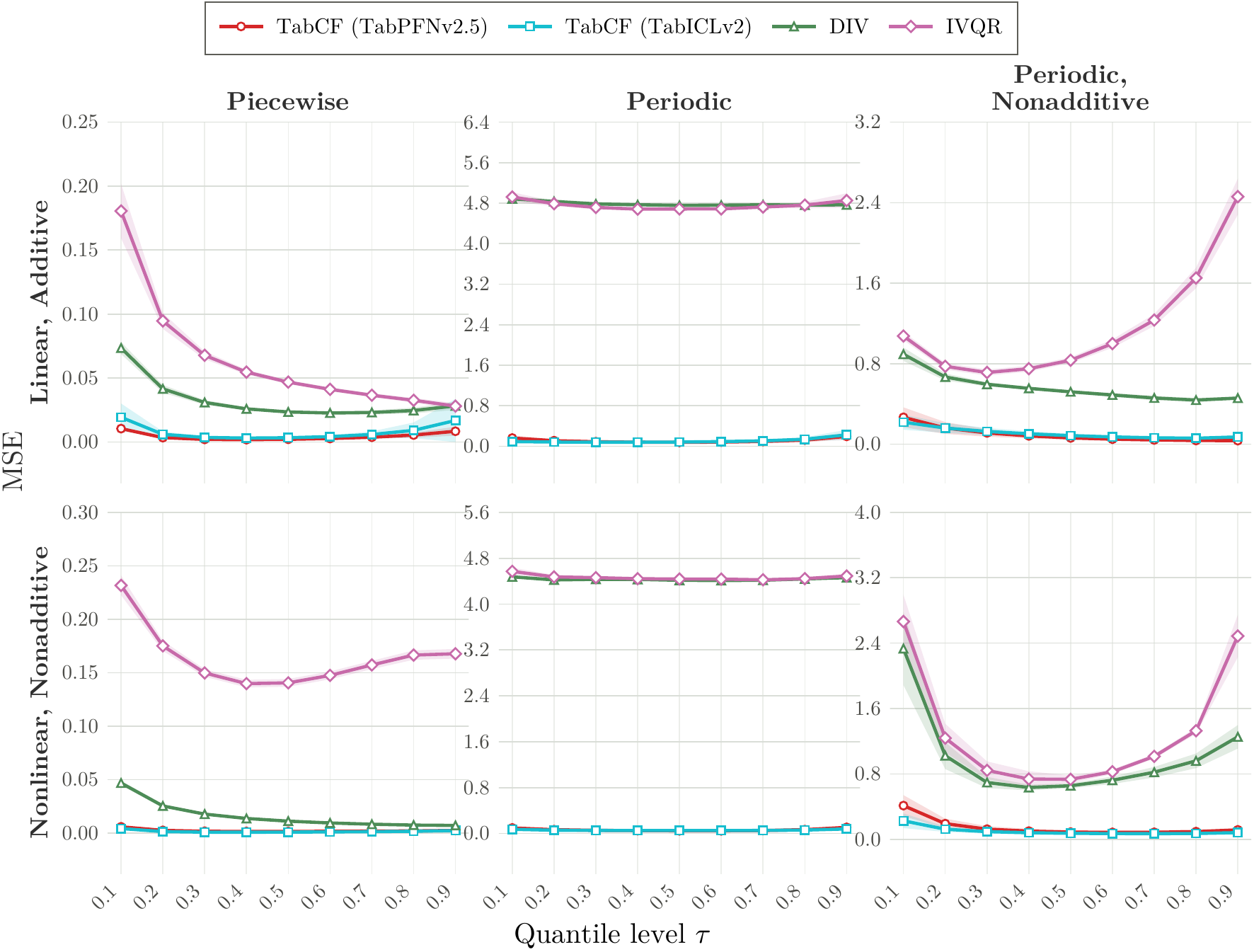}
    \caption{Interventional quantile MSE with one observed pretreatment covariate at $n=4000$. Rows correspond to treatment settings, and columns correspond to outcome settings. Results are averaged over 100 random seeds with standard deviations in shaded regions.}
    \label{fig:app-pretreatment-w-quantile}
\end{figure}

\paragraph{Results.}
Figure~\ref{fig:app-pretreatment-w-quantile} shows that TabCF remains accurate after covariate adjustment. This is also a practical advantage of the TabCF interface, because many existing methods \cite{kookpfister2025dive,saengkyongam2022hsicx,bennett2019deepgmm} do not natively support pretreatment covariates.

\subsection{Experiments on violation of common support condition}
\label{app:uniform-z-common-support}

TabCF requires the common support condition (\ref{CF: common support}), which may not be satisfied in practice. Thus, it is crucial to assess its robustness to violations of (\ref{CF: common support}). To this end, we consider synthetic experiments with a bounded instrument,
$Z\sim\mathrm{Unif}[0,3]$. Apart from replacing the instrument distribution, we use the same settings as the experiments in Section \ref{sec:simulations}. Recall that the treatment models are given as: 

\begin{enumerate}[label=({T\arabic*}), itemsep=3pt, topsep=2pt, parsep=0pt, partopsep=0pt]
\item \textit{Additive linear mean:} $X = Z + H + \varepsilon_X$; 
\item \textit{Quadratic mean with linear scale:} $X = \left(2Z + \tfrac{1}{4}Z^2\right) + (1 + 0.15Z)\,(H+\varepsilon_X)$.
\end{enumerate}

\paragraph{Common support condition (\ref{CF: common support}) is violated.}
Let $
\eta = H+\varepsilon_X \sim \mathcal N(0,2)$.
Because $Z$ is bounded, conditioning on an observed treatment value restricts the possible values of the first-stage disturbance. In the linear case,
\[
    \operatorname{supp}(\eta\mid X=x)=[x-3,x]
    \subsetneq \mathbb R
    =
    \operatorname{supp}(\eta).
\]
In the nonlinear case,
\[
    \operatorname{supp}(\eta\mid X=x)
    =
    \left\{
        \frac{x-2z-z^2/4}{1+0.15z}: z\in[0,3]
    \right\}
    \subsetneq \mathbb R.
\]
Thus, unlike the normal $Z$ design, these bounded-IV setups violate the common support condition in Section~\ref{sec:identification}. The evaluation is nevertheless
restricted to the central treatment region, so the violation is moderate.

\begin{figure}[ht]
    \centering
    \includegraphics[width=0.85\linewidth]{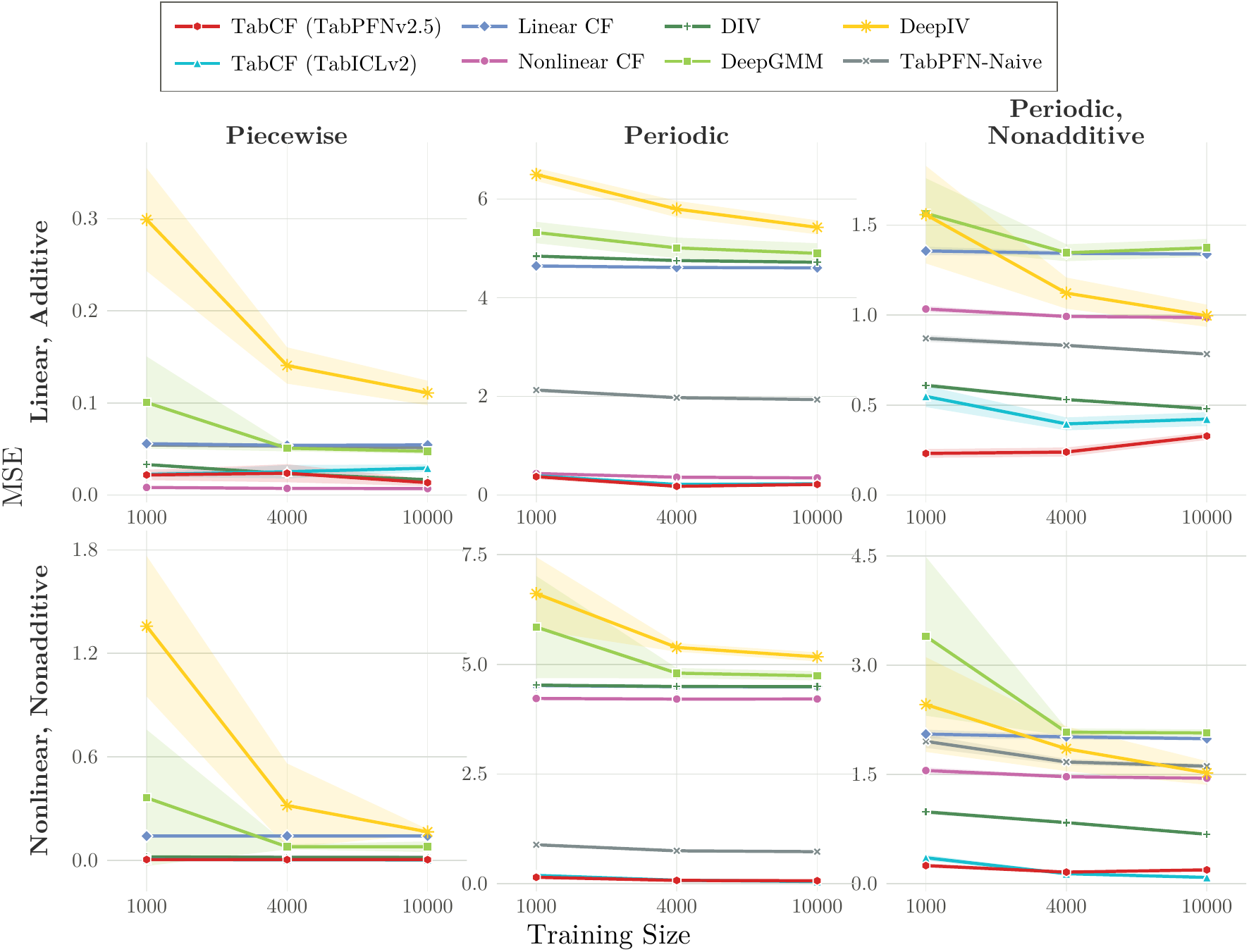}
    \caption{Interventional mean estimation under $Z\sim\mathrm{Unif}[0,3]$. Rows correspond to treatment settings, and columns correspond to outcome settings. Results are averaged over 100 random seeds with standard deviations in shaded regions.}
    \label{fig:app-uniform-z-mean}
\end{figure}

\begin{figure}[h]
    \centering
    \includegraphics[width=0.8\textwidth]{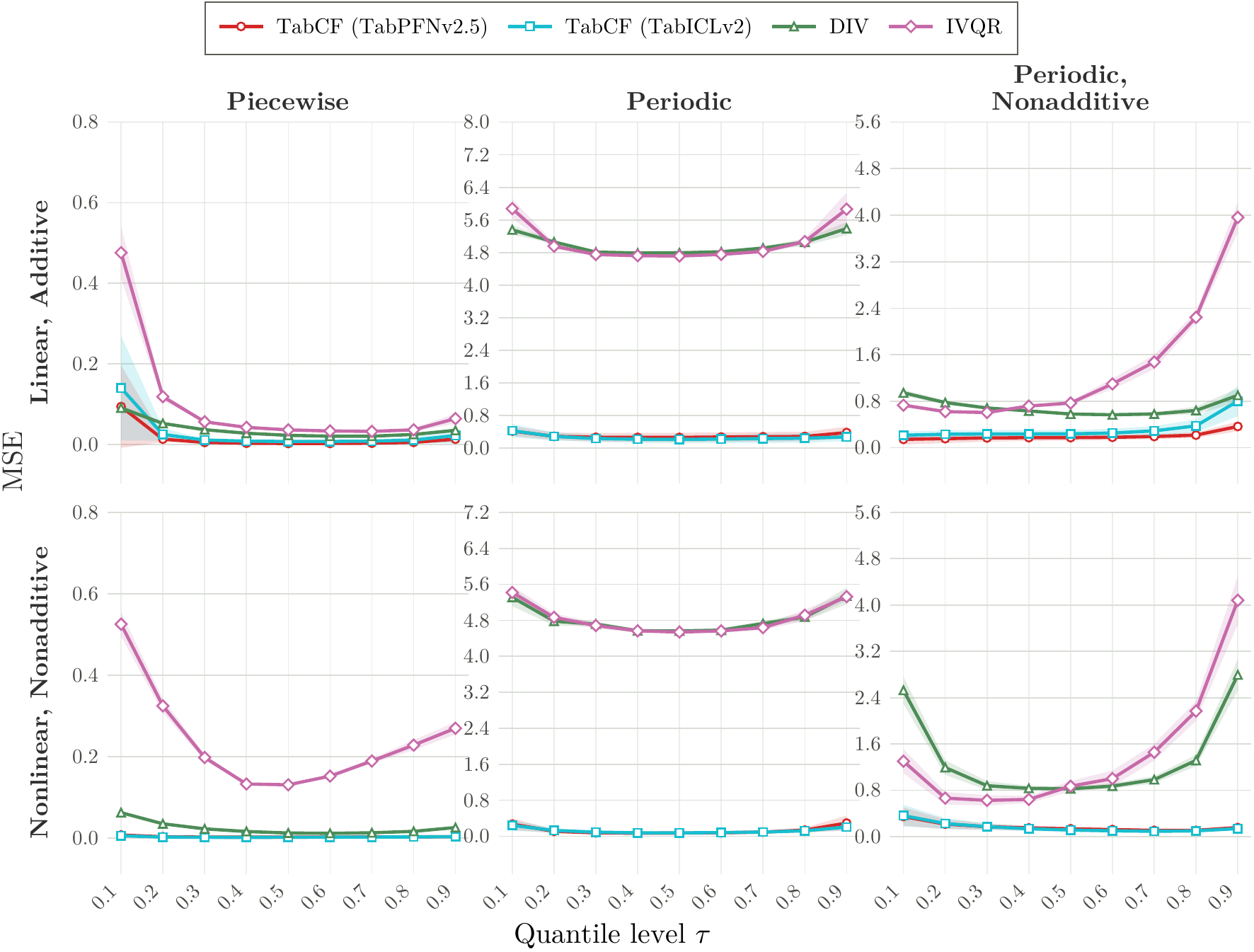}
    \caption{Interventional quantile estimation under $Z\sim\mathrm{Unif}[0,3]$ at $n=4000$. Rows correspond to treatment settings and columns correspond to outcome settings. Results are averaged over 100 random seeds with standard deviations in shaded regions.}
    \label{fig:app-uniform-z-quantile}
\end{figure}

\begin{figure}[!htbp]
    \centering
    \includegraphics[width=\textwidth]{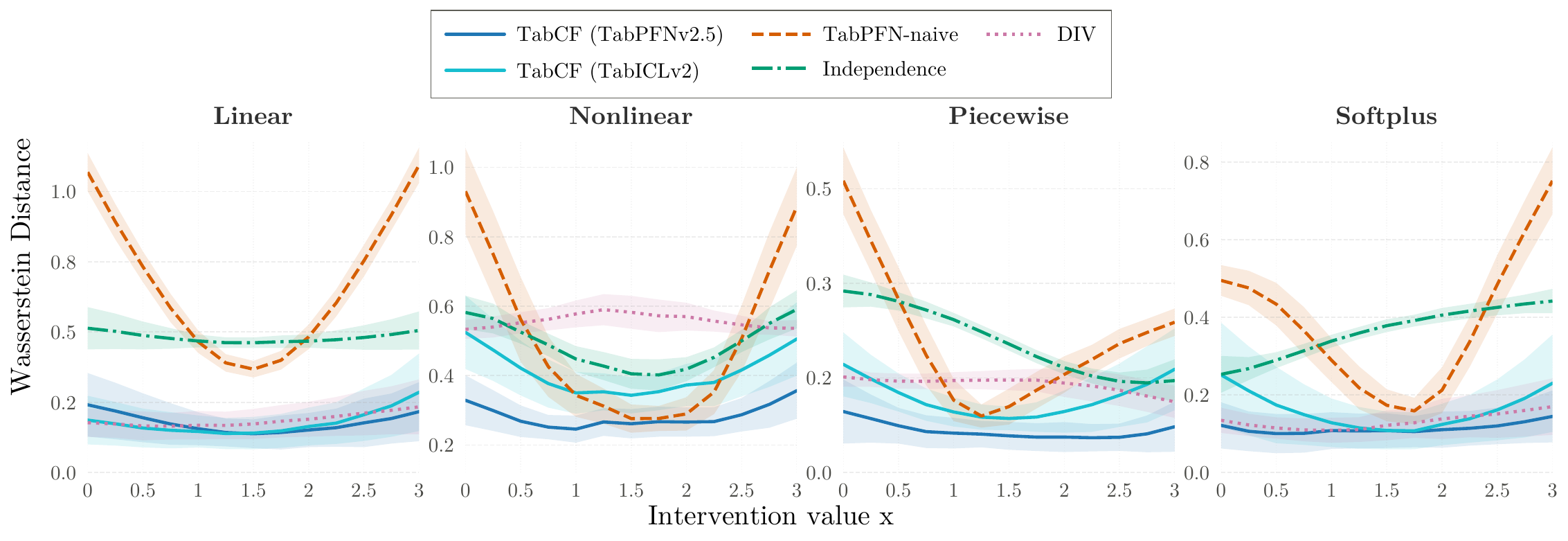}
    \caption{Bivariate-outcome sliced Wasserstein distance under $Z\sim\mathrm{Unif}[0,3]$ when $n=2000$ and $\rho_{\varepsilon}=0.6$. Results are averaged over 100 random seeds with standard deviations in shaded regions.}
    \label{fig:app-uniform-z-multivar}
\end{figure}

\paragraph{Results.}
Figures~\ref{fig:app-uniform-z-mean}--\ref{fig:app-uniform-z-multivar}
are the uniform-$Z$ counterparts of the experiments in
Sections~\ref{sec:sim-results-mean},~\ref{sec:sim-results-quantile}, and~\ref{sec:sim-results-multivariate},
respectively. The results show that {TabCF} remains competitive under this moderate violation of the common support condition,
supporting the empirical robustness of the method beyond the well-specified regime used for
the experiments in Section \ref{sec:simulations}.

\subsection{Experiments on instrument relevance}
\label{app:instrument-relevance}

This section examines how TabCF behaves as the strength of instrument $Z$ varies, especially when the IV strength is weak. 

We use the interventional quantile setting with $n=4000$, $\tau=0.5$, $Z\sim\mathcal N(1.5,0.75^2)$, and modify the six treatment-outcome settings from Section~\ref{sec:sim-results-quantile}. We consider three weak-IV setups indexed by $\kappa\in\{0.05,0.15,0.25\}$ as described below, where smaller $\kappa$ makes $X$ less dependent on $Z$. 
For treatment model (T1), we modify it by setting
\[
    X = 1.5+\kappa(Z-1.5)+H+\varepsilon_X,
\]
while for treatment model (T2), we shrink both the mean and scale toward their marginal normal-$Z$ expectations,
\[
    X =
    \bar m+\kappa\{m(Z)-\bar m\}
    +\{\bar s+\kappa(s(Z)-\bar s)\}(H+\varepsilon_X),
    \quad
    m(z)=2z+\tfrac14 z^2,\quad s(z)=1+0.15z.
\]
Here $\bar m$ and $\bar s$ denote the corresponding marginal expectations under the normal-$Z$ distribution.

\begin{table}[htbp]
    \centering
    \caption{MSE of interventional median ($\tau=0.5$) estimation. Results are averaged over 100 random seeds with standard deviations in parentheses.}
    \label{tab:weakiv_median_mse}
    \small
    \begin{tabular}{llcccc}
        \toprule
        \textbf{$\kappa$} & \textbf{Setting} & \textbf{TabCF (TabPFNv2.5)} & \textbf{TabCF (TabICLv2)} & \textbf{DIV} & \textbf{IVQR} \\
        \midrule
        0.05 & (T1,O1) & \textbf{0.02 (0.01)} & 0.08 (0.04) & \textbf{0.02 (0.01)} & 0.06 (0.39) \\
        0.05 & (T1,O2) & 5.51 (0.63) & \textbf{3.86 (0.84)} & 7.32 (0.24) & 5.33 (6.41) \\
        0.05 & (T1,O3) & 0.86 (0.16) & 1.89 (0.67) & \textbf{0.66 (0.13)} & 1.48 (28.05) \\
        0.05 & (T2,O1) & 0.01 (0.01) & 0.05 (0.04) & \textbf{0.00 (0.01)} & 0.09 (0.03) \\
        0.05 & (T2,O2) & \textbf{2.53 (0.57)} & 4.08 (0.81) & 7.63 (0.53) & 5.28 (2.52) \\
        0.05 & (T2,O3) & \textbf{0.59 (0.43)} & 1.73 (1.89) & 0.97 (0.10) & 1.71 (1.67) \\
        \midrule
        0.15 & (T1,O1) & 0.03 (0.02) & 0.03 (0.08) & \textbf{0.02 (0.00)} & 0.03 (0.01) \\
        0.15 & (T1,O2) & 2.81 (0.71) & \textbf{2.22 (0.66)} & 6.09 (0.52) & 4.58 (0.39) \\
        0.15 & (T1,O3) & 0.67 (0.07) & 1.19 (0.41) & \textbf{0.62 (0.09)} & 0.92 (0.39) \\
        0.15 & (T2,O1) & \textbf{0.00 (0.01)} & 0.01 (0.03) & \textbf{0.00 (0.00)} & 0.09 (0.01) \\
        0.15 & (T2,O2) & \textbf{0.54 (0.21)} & 1.66 (0.48) & 4.91 (0.27) & 4.67 (0.29) \\
        0.15 & (T2,O3) & \textbf{0.30 (0.19)} & 0.70 (0.82) & 0.94 (0.09) & 1.02 (0.24) \\
        \midrule
        0.25 & (T1,O1) & \textbf{0.02 (0.02)} & \textbf{0.02 (0.02)} & \textbf{0.02 (0.00)} & 0.03 (0.01) \\
        0.25 & (T1,O2) & 1.88 (0.67) & \textbf{1.57 (0.48)} & 4.95 (0.39) & 4.42 (0.16) \\
        0.25 & (T1,O3) & \textbf{0.48 (0.09)} & 0.86 (0.30) & 0.53 (0.07) & 0.96 (0.15) \\
        0.25 & (T2,O1) & \textbf{0.00 (0.00)} & 0.01 (0.01) & \textbf{0.00 (0.00)} & 0.08 (0.01) \\
        0.25 & (T2,O2) & \textbf{0.18 (0.16)} & 0.59 (0.22) & 4.50 (0.09) & 4.55 (0.15) \\
        0.25 & (T2,O3) & \textbf{0.10 (0.22)} & 0.31 (0.27) & 0.91 (0.09) & 0.97 (0.14) \\
        \bottomrule
    \end{tabular}
\end{table}

\paragraph{Results.}
Table~\ref{tab:weakiv_median_mse} shows that weaker relevance mainly affects the harder periodic outcome settings. Errors generally decrease as $\kappa$ increases, most visibly for O2 and O3. Across these weak-IV regimes, TabCF with TabPFNv2.5 is competitive in nearly all settings and are often the best-performing methods in the nonlinear T2 rows, while DIV remains strong in the simplest O1 cases and in some T1/O3 settings. IVQR is comparatively stable in the easiest cases but is less accurate on the harder weak-IV designs. Overall, TabCF degrades gracefully as instrument relevance weakens. This result also suggests TabPFNv2.5 is more robust than TabICLv2 for TabCF in weak relevance regime.

\section{Implementation notes}
\label{app:benchmarks}

\subsection{TabCF}

\paragraph{Interventional means.}

For the mean experiments in Section~\ref{sec:sim-results-mean}, {TabCF} is implemented as the two-stage plug-in procedure from Section~\ref{sec:tabpfn}. The first stage fits a probabilistic tabular backbone to $X$ given $Z$ and forms $\widehat V_i=\widehat F_{X\mid Z}(X_i\mid Z_i)$ using the backbone's predictive CDF. The second stage fits the same type of backbone to $Y$ given $(X,\widehat V)$ and evaluates the structural mean on the intervention grid by averaging over $v\in[0,1]$. In light of the common-support condition in Section~\ref{sec:identification}, the reported MSE is computed only on held-out intervention values whose $X$ values fall in the empirical $0.05$--$0.95$ quantile range of the test-set treatment distribution; the training data used to fit the two stages are not trimmed. The TabPFN and TabICL variants differ only in this predictive backbone; no design-specific hyperparameter search is performed.

\paragraph{Interventional quantiles.}
For Section~\ref{sec:sim-results-quantile}, the implementation estimates the full interventional CDF before computing quantiles. After constructing $\widehat V$ as above, we fit a conditional CDF model for $Y\mid X,\widehat V$, evaluate $\widehat F_{Y\mid X,V}(y\mid x,v)$ on a fixed $(x,y)$ grid, integrate the CDF values over the same $v$ quadrature rule, and then invert the resulting monotone CDF numerically for each requested $\tau$. This keeps the quantile estimator tied to a single estimated distribution rather than fitting a separate quantile regression at each level. Because each intervention value requires repeated CDF evaluations across both outcome and control-function grids, the quantile experiments use 200 treatment values for evaluation. These treatment values are chosen as evenly spaced empirical quantiles over the central $0.05$--$0.95$ range of the held-out treatment distribution, matching the same interior-support convention used in the mean experiments while keeping the distributional computation feasible.

For Section~\ref{sec:sim-results-multivariate}, {TabCF} first applies the scalar distributional estimator componentwise to obtain marginal interventional CDFs for $Y_1$ and $Y_2$. The joint law is then assembled with an $x$-invariant copula fitted from pseudo-uniform scores derived from the estimated marginals. The same marginal-fitting interface is used for the TabPFN, TabICL, and naive marginal variants; the independence baseline keeps the {TabCF} marginals but fixes the copula to $C(u_1,u_2)=u_1u_2$. For the bivariate experiments, we evaluate the joint law at 13 equally spaced intervention values over $[0,3]$, ensuring that the target interventions remain within the central design region used to generate the data. This coarser grid also controls computation, since each intervention value requires marginal distribution evaluation for both outcomes, copula-based sampling or CDF evaluation, and comparison with the oracle joint law.

\subsection{Competing methods}

This appendix summarizes the benchmark methods used in Sections~\ref{sec:sim-results-mean} and~\ref{sec:sim-results-quantile}.
For each method, we indicate, when relevant, whether it is primarily designed for interventional mean estimation,
distribution/quantile estimation, or both. We list only methods that appear in the reported benchmark tables and figures; related-work methods that are discussed elsewhere but not benchmarked here are omitted unless needed for clarification.

% \paragraph{Classical linear IV (2SLS / linear IV).}
% Two-stage least squares (2SLS) and closely related linear-IV estimators fit a first-stage linear projection of $X$ on $Z$
% and then regress $Y$ on the fitted treatment. These methods are standard mean-effect baselines under linearity and are not
% designed to recover the full interventional distribution or quantile curve.

\paragraph{Two-stage residual inclusion.}
Two-stage residual inclusion (2SRI) regression was developed in econometrics for interventional mean estimation
based on augmenting the outcome regression with a control variable that accounts for endogeneity
\citep{terza20082SRI,guo2016control}. Following the DIV benchmark implementation, we include both
a linear CF baseline and a nonlinear CF baseline; the nonlinear version uses natural cubic splines for basis expansion.
The implementations use R's \texttt{lm} routine, together with \texttt{splines::ns} for the nonlinear variant. We use the
benchmark defaults, including \texttt{df = 5} for the spline basis.

\paragraph{DeepIV.}
DeepIV \citep{hartford2017deepiv} estimates the conditional distribution $p(X\mid Z)$ in the first stage and then learns a
second-stage response model by integrating over that estimated treatment distribution. It is primarily designed for
interventional mean estimation rather than full distributional recovery. The method repository is
\url{https://github.com/jhartford/DeepIV}; our benchmark uses the \texttt{econml.iv.nnet.DeepIV} implementation from
\url{https://github.com/py-why/EconML}. We use the fixed default settings in our benchmark wrapper, including the specified
feed-forward Keras architectures, training options, and restart count.

\paragraph{DeepGMM.}
DeepGMM \citep{bennett2019deepgmm} uses neural networks within a generalized method of moments (GMM) objective to learn a structural
function satisfying the moment conditions, targeting mean effects. The method repository is
\url{https://github.com/CausalML/DeepGMM}. We call the released \texttt{ToyModelSelectionMethod} with its built-in model
selection and training defaults.

\paragraph{DIV.}
Distributional Instrumental Variable (DIV) \citep{holovchak2025div} is a generative-model-based IV method that directly targets
the interventional distribution in nonlinear settings. In our comparisons, DIV is the distributional baseline used for both
mean and quantile functionals: means are computed from the estimated interventional distribution, and quantiles are obtained
by inverting the estimated interventional CDF. The software is the R package \texttt{DistributionIV}; its CRAN page is
\url{https://cran.r-project.org/package=DistributionIV}, and a read-only CRAN GitHub mirror is
\url{https://github.com/cran/DistributionIV}. We use the package/default benchmark settings for the neural DIV fit
(e.g., default noise dimensions, layer count, epochs, and learning rate in the simulation wrapper).

\paragraph{IVQR.}
Instrumental variable quantile regression (IVQR) \citep{chernozhukov2005ivqr} targets quantile treatment effects using IVs.
It is a standard quantile-focused econometric baseline in our interventional quantile experiments. The R package repository is
\url{https://github.com/yuchang0321/IVQR}. We use the package's default IVQR estimation choices together with the benchmark
wrapper's automatic coefficient grid.

\paragraph{TabPFN-naive.}
{TabPFN-naive} uses the TabPFNv2.5 predictive backbone as {TabCF} but treats the task as a direct regression
of $Y$ on $X$, without any IV adjustment or control-function step. It therefore serves as an endogeneity-ignorant reference
that isolates the value of the IV correction rather than the predictive backbone itself. The TabPFN repository is
\url{https://github.com/PriorLabs/TabPFN}. We use the default pretrained \texttt{TabPFNRegressor} configuration in this
naive regression wrapper.

\section{Details of real data examples}
\label{app:real-data-details}

This appendix provides background details for the four real-data applications summarized in
Section~\ref{sec:real-data}. The purpose of these applications is not to provide a definitive
empirical re-analysis of each original study. Instead, we use them as standardized scalar-IV
benchmarks for assessing whether TabCF produces economically interpretable and appropriately
regularized causal curves on small or moderate-sized tabular data. For comparability across
applications, we retain one excluded instrument, one endogenous treatment, and one outcome in
each dataset. Continuous-treatment applications are evaluated on a grid over the observed
treatment support, while the schooling application is evaluated only at observed schooling levels
because education is ordered and discrete.

\paragraph{AJR.}
The colonial-origins application follows the Acemoglu-Johnson-Robinson design, in which
historical settler mortality is used as an instrument for institutional quality~\citep{acemoglu2001colonial}. There are $n=64$ countries.
The economic mechanism is that European colonization strategies differed across disease
environments: where settler mortality was high, colonial powers were less likely to establish
inclusive institutions and secure property-rights protections. The treatment is therefore a measure
of institutional quality, and the outcome is long-run income measured by log GDP per capita. Prior linear-IV
analyses and the recent DIV application both suggest a positive, approximately linear relationship
between institutional quality and income~\citep{holovchak2025div}. Thus, a desirable flexible
method should recover the positive direction of the effect without introducing unsupported
nonlinearities.

\paragraph{Fulton Fish.}
The Fulton Fish Market application is a canonical example in which price is endogenous
because market-clearing prices respond to both supply and demand shocks \cite{graddy1995fish}. It consists of $n=111$ days of transaction-level data. Weather conditions
provide a natural source of exogenous variation in supply: rougher sea conditions affect the
quantity of fish brought to market and therefore prices, while plausibly having no direct effect on
consumer demand conditional on price. We use this supply-side variation to estimate how the log
quantity sold changes with log average price. This benchmark is useful because economic theory
gives a clear qualitative reference: demand should decrease as price rises, while the small sample
leaves room for flexible methods to overfit local variation. On the same benchmark, the
interventional quantile comparison among TabCF (TabPFNv2.5), DIV, and IVQR yields
uniformly downward-sloping structural quantile curves. TabCF remains the smoothest across
quantile levels, while DIV and IVQR produce steeper and more quantile-sensitive demand
responses.

\paragraph{Card.}
The Card application ($n=1484$) uses geographic proximity to a four-year college as an instrument for
educational attainment~\citep{card1995geographic}. The instrument shifts the cost of obtaining
additional schooling, and the outcome is the log wage in adulthood. Unlike the continuous-treatment
benchmarks above, this application has a binary instrument and an ordered discrete treatment.
Moreover, the original Card-style specifications typically include a richer set of pre-treatment
covariates than the scalar-IV benchmark used here. We therefore interpret this application as a
stress test for regularized curve estimation under a familiar empirical design, rather than as a
fully specified replication of the original returns-to-schooling analysis. The fitted curve should be
read primarily as a comparison of how the methods regularize the wage--schooling relationship
over the observed education support.

\paragraph{CigarettesSW.}
The cigarette-demand application uses state-level ($n=48\times 2$) variation in cigarette taxes to instrument for real
cigarette prices \cite{stockwatson2007introduction}. The outcome is log packs per capita, and the treatment is log real price. Since
cigarette prices may reflect both policy variation and unobserved demand conditions, tax variation
provides the excluded source of price variation used for IV estimation. This application serves as
a second demand benchmark: the economically relevant qualitative question is whether the fitted
interventional mean curve is consistent with lower consumption at higher prices while remaining
stable in a small state-level panel.

\end{document}